\documentclass{article}

\usepackage{fullpage}
\usepackage{amsmath,amssymb,amsthm}
\usepackage{algorithm,algorithmic}
\usepackage{times}
\usepackage{graphicx}
\usepackage{subfigure}

\usepackage{hyperref}
\usepackage{xspace}

\usepackage{stmaryrd}

\newcommand{\nemsis}{\textbf{NEMSIS}\xspace}
\newcommand{\nemsisns}{\textbf{NEMSIS-NS}\xspace}
\newcommand{\scan}{\textbf{SCAN}\xspace}
\newcommand{\scanns}{\textbf{SCAN-NS}\xspace}
\newcommand{\scanfc}{\textbf{CAN}\xspace}
\newcommand{\kld}{\text{KLD}\xspace}
\newcommand{\tpr}{\text{TPR}\xspace}
\newcommand{\tnr}{\text{TNR}\xspace}
\newcommand{\ba}{\text{BA}\xspace}
\newcommand{\bakld}{\text{BAKLD}\xspace}
\newcommand{\nss}{\text{NSS}\xspace}
\newcommand{\fp}{\text{FP}\xspace}
\newcommand{\fn}{\text{FN}\xspace}

\newcommand{\class}{\text{class}\xspace}
\newcommand{\quant}{\text{quant}\xspace}

\newcommand{\<}{\leftarrow}

\newcommand{\Ind}[1]{{\mathbb I}\bs{{#1}}}

\newcommand{\R}{{\mathbb R}}

\newcommand{\T}{{\cal T}}
\renewcommand{\P}{{\mathbf P}}

\newcommand{\pdsgd}{{\bfseries SPADE}\xspace}
\newcommand{\am}{{\bfseries AMP}\xspace}
\newcommand{\amsgd}{{\bfseries STAMP}\xspace}

\newcommand{\1}{{\mathbf 1}}

\newcommand{\X}{{\mathcal X}}
\newcommand{\Y}{{\mathcal Y}}

\newcommand{\A}{{\mathcal A}}

\newcommand{\D}{\mathcal D}

\newcommand{\W}{{\mathcal W}}

\newcommand{\barw}{{\overline \w}}
\newcommand{\z}{{\mathbf z}}

\newcommand{\Pf}{{\mathcal P}}

\renewcommand{\vec}[1]{{\mathbf{#1}}}

\newcommand{\vecs}{\vec{s}}

\newcommand{\veczero}{\vec{0}}
\newcommand{\vz}{\veczero}

\newcommand{\x}{\vec{x}}

\newcommand{\w}{\vec{w}}

\newcommand{\y}{\vec{y}}
\newcommand{\rew}{\vec{r}}
\newcommand{\qew}{\vec{q}}

\newcommand{\bc}[1]{\left\{{#1}\right\}}
\newcommand{\br}[1]{\left({#1}\right)}
\newcommand{\bs}[1]{\left[{#1}\right]}
\newcommand{\abs}[1]{\left| {#1} \right|}
\newcommand{\norm}[1]{\left\| {#1} \right\|}

\newcommand{\bsd}[1]{\left\llbracket{#1}\right\rrbracket}

\newcommand{\etal}{\emph{et al}\ }
\renewcommand{\O}[1]{{\cal O}\br{{#1}}}
\newcommand{\softO}[1]{\widetilde{\cal O}\br{{#1}}}

\newcommand{\E}[1]{{\mathbb E}\bsd{{#1}}}

\newcommand{\EE}[2]{\underset{#1}{\mathbb E}\bsd{{#2}}}

\newcommand{\Prr}[2]{\underset{#1}{\text{Pr}}\bs{{#2}}}
\newcommand{\ip}[2]{\left\langle{#1},{#2}\right\rangle}

\newcommand{\valpha}{\text{\boldmath$\mathbf{\alpha}$}}
\newcommand{\vbeta}{\text{\boldmath$\mathbf{\beta}$}}
\newcommand{\vgamma}{\text{\boldmath$\mathbf{\gamma}$}}

\newtheorem{lem}{Lemma}
\newtheorem{thm}[lem]{Theorem}

\newtheorem{clm}[lem]{Claim}
\newtheorem{defn}[lem]{Definition}

\newcommand{\svmperf}{\text{SVM$^{\text{\it perf}}$}\xspace}
\newcommand{\email}[1]{\texttt{#1}}

\allowdisplaybreaks

\title{Online Optimization Methods for\\the Quantification Problem\thanks{A short version of this manuscript will appear in the proceedings of the 22nd ACM SIGKDD Conference on Knowledge Discovery and Data Mining, KDD 2016.}}

\author{
Purushottam Kar\\IIT Kanpur, India\\\email{purushot@cse.iitk.ac.in}
\and
Shuai Li\thanks{This work was done while Shuai Li was a research associate at QCRI-HBKU.}\\University of Insubria, Italy\\\email{shuaili.sli@gmail.com}
\and
Harikrishna Narasimhan\\Harvard University, USA\\\email{hnarasimhan@g.harvard.edu}
\and
Sanjay Chawla\thanks{Sanjay Chawla is on leave from University of Sydney.}\\QCRI-HBKU, Qatar\\\email{schawla@qf.org.qa}
\and
Fabrizio Sebastiani\thanks{Fabrizio Sebastiani is on leave from Consiglio Nazionale delle Ricerche, Italy.}\\QCRI-HBKU, Qatar\\\email{fsebastiani@qf.org.qa}    
}

\date{}

\begin{document}

\maketitle

\begin{abstract}
The estimation of class prevalence, i.e., the fraction of a population that belongs to a certain class, is a very useful tool in data analytics and learning, and finds applications in many domains such as sentiment analysis, epidemiology, etc. For example, in sentiment analysis, the objective is often not to estimate whether a specific text conveys a positive or a negative sentiment, but rather estimate the overall distribution of positive and negative sentiments during an event window. A popular way of performing the above task, often dubbed \emph{quantification}, is to use supervised learning to train a prevalence estimator from labeled data.

Contemporary literature cites several performance measures used to measure the success of such prevalence estimators. In this paper we propose the first online stochastic algorithms for \emph{directly} optimizing these quantification-specific performance measures. We also provide algorithms that optimize \emph{hybrid} performance measures that seek to balance quantification and classification performance. Our algorithms present a significant advancement in the theory of multivariate optimization and we show, by a rigorous theoretical analysis, that they exhibit optimal convergence. We also report extensive experiments on benchmark and real data sets which demonstrate that our methods significantly outperform existing optimization techniques used for these performance measures.
\end{abstract}

\section{Introduction}
\label{intro}

\noindent \emph{Quantification} \cite{Forman:2008kx} is defined as the task of estimating the prevalence (i.e., relative frequency) of the classes of interest in an unlabeled set, given a training set of items labeled according to the same classes. Quantification finds its natural application in contexts characterized by \emph{distribution drift}, i.e., contexts where the training data may not exhibit the same class prevalence pattern as the test data. This phenomenon may be due to different reasons, including the inherent non-stationary character of the context, or class bias that affects the selection of the training data. 

A na\"ive way to tackle quantification is via the ``classify and count'' (CC) approach, i.e., to classify each unlabeled item independently and compute the fraction of the unlabeled items that have been attributed to each class. However, a good classifier does not necessarily lead to a good quantifier: assuming the binary case, even if the sum (FP + FN) of the false positives and false negatives is comparatively small, bad quantification accuracy might result if FP and FN are significantly different (since perfect quantification coincides with the case FP = FN). This has led researchers to study quantification as a task in its own right, rather than as a byproduct of classification.

The fact that quantification is not just classification in disguise can also be seen by the fact that evaluation measures different from those for classification (e.g., $F_{1}$, AUC) need to be employed. Quantification actually amounts to computing how well an estimated class distribution $\hat{p}$ fits an actual class distribution $p$ (where for any class $c \in \mathcal C$, $p(c)$ and $\hat p(c)$ respectively denote its true and estimated prevalence); as such, the natural way to evaluate the quality of this fit is via a function from the class of \emph{$f$-divergences} \cite{Csiszar:2004fk}, and a natural choice from this class (if only for the fact that it is the best known $f$-divergence) is the \emph{Kullback-Leibler Divergence} (\kld), defined as
\begin{equation}
  \label{eq:KLD}
  \kld(p,\hat{p}) = \sum_{c\in \mathcal{C}} 
  p(c)\log\frac{p(c)}{\hat{p}(c)}
\end{equation}
\noindent Indeed, \kld is the most frequently used measure for evaluating quantification (see e.g., \cite{Barranquero:2015fr,Esuli:2015gh,Forman:2008kx,Gao:2015ly}).  Note that \kld is non-decomposable, i.e., the error we make by estimating $p$ via $\hat{p}$ cannot be broken down into item-level errors. This is not just a feature of \kld, but an inherent feature of \emph{any} measure for evaluating quantification. In fact, how the error made on a given unlabeled item impacts the overall quantification error depends on how the other items have been classified\footnote{For the sake of simplicity, we assume here that quantification is to be tackled in an \emph{aggregative} way, i.e., the classification of individual items is a necessary intermediate step for the estimation of class prevalences. Note however that this is not necessary; non-aggregative approaches to quantification may be found in \cite{Gonzalez-Castro:2013fk,King:2008fk}.}; e.g., if $\fp>\fn$ for the other unlabeled items, then generating an additional false negative is actually \emph{beneficial} to the overall quantification accuracy, be it measured via \kld or via any other function. 

The fact that \kld is the measure of choice for quantification and that it is non-decomposable, has lead to the use of structured output learners, such as \svmperf \cite{Joachims05}, that allow a direct optimization of non-decomposable functions; the approach of Esuli and Sebastiani \cite{Esuli:2010fk,Esuli:2015gh} is indeed based on optimizing \kld using \svmperf. However, that minimizing \kld\ (or $|\fp-\fn|$, or any ``pure'' quantification measure) should be the only objective for quantification regardless of the value of $\fp+\fn$ (or any other classification measure), is fairly paradoxical. Some authors \cite{Barranquero:2015fr,Milli:2013fk} have observed that this might lead to the generation of unreliable quantifiers (i.e., systems with good quantification accuracy but bad or very bad classification accuracy), and have, as a result, championed the idea of optimizing ``multi-objective'' measures that combine quantification accuracy with classification accuracy. Using a decision-tree-like approach, \cite{Milli:2013fk} minimizes $|\fp^{2}-\fn^{2}|$, which is the product of $|\fn-\fp|$, a measure of quantification error, and $(\fn+\fp)$, a measure of classification error; \cite{Barranquero:2015fr} also optimizes (using \svmperf) a measure that combines quantification and classification accuracy.

While \svmperf does provide a recipe for optimizing general performance measures, it has serious limitations. \svmperf is not designed to directly handle applications where large streaming data sets are the norm. \svmperf also does not scale well to multi-class settings, and the time required by the method is exponential in the number of classes. 

In this paper we develop stochastic methods for optimizing a large family of popular quantification performance measures. Our methods can effortlessly work with streaming data and scale to very large datasets, offering training times up to an order of magnitude faster than other approaches such as \svmperf.

\section{Related Work}
\label{relatedwork}

\noindent \textbf{Quantification methods.}
The quantification methods that have been proposed over the years can be broadly classified into two classes, namely aggregative and non-aggregative methods. While \emph{aggregative} approaches perform quantification by first classifying individual items as an intermediate step, \emph{non-aggregative} approaches do not require this step, and estimate class prevalences holistically. Most methods, such as those of \cite{Barranquero:2015fr,Bella:2010kx,Esuli:2015gh,Forman:2008kx,Milli:2013fk}, fall in the former class, while the latter class has few representatives \cite{Gonzalez-Castro:2013fk,King:2008fk}. 

Within the class of aggregative methods, a further distinction can be made between methods, such as those of \cite{Bella:2010kx,Forman:2008kx}, that first use general-purpose learning algorithms and then post-process their prevalence estimates to account for their estimation biases, and methods (which we have already hinted at in Section \ref{intro}) that instead use learning algorithms explicitly devised for quantification \cite{Barranquero:2015fr,Esuli:2015gh,Milli:2013fk}. In this paper we focus the latter class of methods.\\ 


\noindent \textbf{Applications of quantification.}
From an application perspective, quantification is especially useful in fields (such as social science, political science, market research, and epidemiology) which are inherently interested in aggregate data, and care little about individual cases. Aside from applications in these fields \cite{Hopkins:2010fk,King:2008fk}, quantification has also been used in contexts as diverse as natural language processing \cite{Chan2006}, resource allocation \cite{Forman:2008kx}, tweet sentiment analysis \cite{Gao:2015ly}, and the veterinary sciences \cite{Gonzalez-Castro:2013fk}. Quantification has independently been studied within statistics \cite{Hopkins:2010fk,King:2008fk}, machine learning \cite{Balikas:2015mz,du-Plessis:2012nr,Saerens:2002uq}, and data mining \cite{Esuli:2015gh,Forman:2008kx}. Unsurprisingly, given this varied literature, quantification also goes under different names, such as  \emph{counting} \cite{Lewis95}, \emph{class
probability re-estimation} \cite{Alaiz-Rodriguez:2011fk}, \emph{class
prior estimation} \cite{Chan2006}, and \emph{learning of class balance} \cite{du-Plessis:2012nr}. 

In some applications of quantification, the estimation of class prevalences is not an end in itself, but is rather used to improve the accuracy of other tasks such as classification. For instance, Balikas et al.\ \cite{Balikas:2015mz} use quantification for model selection in supervised learning, by tuning hyperparameters that yield the best quantification accuracy on validation data; this allows hyperparameter tuning to be performed without incurring the costs inherent to $k$-fold cross-validation. Saerens et al.\ \cite{Saerens:2002uq}, followed by other authors \cite{Alaiz-Rodriguez:2011fk,Xue:2009uq,Zhang:2010kx}, apply quantification to customize a trained classifier to the class prevalence exhibited in the test set, with the goal of improving classification
accuracy on unlabeled data exhibiting a class distribution different from that of the training set. The work of Chan and Ng \cite{Chan2006} may be seen as a direct application of this notion, as they use quantification to tune a word sense disambiguator to the estimated sense priors of the test set. Their work can also be seen as an instance of \emph{transfer learning} (see e.g., \cite{Pan:2012fk}), since their goal is to adapt a word sense disambiguation algorithm to a domain different from the one the algorithm was trained upon.\\

\noindent \textbf{Stochastic optimization.} As discussed in Section~\ref{intro}, our goal in this paper is to perform quantification by directly optimizing, in an online stochastic setting, specific performance measures for the quantification problem. While recent advances have seen much progress in efficient methods for online learning and optimization in full information and bandit settings \cite{online-batch-single,shuai14icml,log-regret,Shalev-ShwartzSSS2009}, these works frequently assume that the optimization objective, or the notion of regret being considered is decomposable and can be written as a sum or expectation of losses or penalties on individual data points. However, performance measures for quantification have a multivariate and complex structure, and do not have this form.

There has been some recent progress \cite{KarNJ14,NarasimhanKJ2015} towards developing stochastic optimization methods for such non-decomposable measures. However, these approaches do not satisfy the needs of our problem. The work of Kar et al.\ \cite{KarNJ14} addresses the problem of optimizing structured \svmperf-style objectives in a streaming fashion, but requires the maintenance of large buffers and, as a result, offers poor convergence. The work of Narasimhan et al.\ \cite{NarasimhanKJ2015} presents online stochastic methods for optimizing performance measures that are concave or pseudo-linear in the canonical confusion matrix of the predictor. However, their method requires the computation of gradients of the Fenchel dual of the performance measures, which is difficult for the quantification performance measures that we study, that have a nested structure. Our methods extend the work of \cite{NarasimhanKJ2015} and provide convenient routines for optimizing the more complex performance measures used for evaluating quantification.

\section{Problem Setting}
\label{formulation}

\noindent For the sake of simplicity, in this paper we will restrict our analysis to binary classification problems and linear models. We will denote the space of feature vectors by $\X \subset \R^d$ and the label set by $\Y = \bc{-1,+1}$. We shall assume that data points are generated according to some fixed but unknown distribution $\D$ over $\X\times\Y$. We will denote the proportion of positives in the population by $p := \Prr{(\x,y)\sim\D}{y = +1}$. Our algorithms, at training time, will receive a set of $T$ training points sampled from $\D$, which we will denote by $\T = \bc{(\x_1,y_1),\ldots,(\x_T,y_T)}$. 

As mentioned above, we will present our algorithms and analyses for learning a linear model over $\X$. We will denote the model space by $\W \subseteq \R^d$ and let $R_\X$ and $R_\W$ denote the radii of domain $\X$ and model space $\W$, respectively. However, we note that our algorithms and analyses can be extended to learning non-linear models by use of kernels, as well as to multi-class quantification problems. However, we postpone a discussion of these extensions to an expanded version of this paper. 

Our focus in this work shall be the optimization of quantification-specific performance measures in online stochastic settings. We will concentrate on performance measures that can be represented as functions of the confusion matrix of the classifier. In the binary setting, the confusion matrix can be completely described in terms of the true positive rate (TPR) and the true negative rate (TNR) of the classifier. However, initially we will develop algorithms that use \emph{reward functions} as surrogates of the TPR and TNR values. This is done to ease algorithm design and analysis, since the TPR and TNR values are count-based and form non-concave and non-differentiable estimators. The surrogates we will use will be concave and almost-everywhere differentiable. More formally, we will use a reward function $r$ that assigns a \emph{reward} $r(\hat y,y)$ to a prediction $\hat{y} \in \R$ for a data point, when the true label for that data point is $y \in \Y$. Given a reward function $r$, a model $\w \in \W$, and a data point $(\x,y) \in \X \times \Y$, we will use
\begin{align*}
r^+(\w; \x, y) &= \frac{1}{p}\cdot r(\w^\top\x, y)\cdot\1(y = 1)\\
r^-(\w;\x, y) &= \frac{1}{1-p}\cdot r(\w^\top\x, y)\cdot\1(y=-1)
\end{align*}
\noindent to calculate rewards on positive and negative points. The average or expected value of these rewards will be treated as surrogates of TPR and TNR respectively. Note that since $\EE{(\x,y)}{r^+(\w; \,\x, y)}$ = $\EE{(\x,y)}{r(\w^\top\x,y)|y = 1}$, setting $r$ to $r^{0\text{-}1}(\hat{y}, y)$ = $\Ind{y\cdot\hat{y} > 0}$, i.e. the classification accuracy function, yields $\EE{(\x,y)}{r^+(\w; \,\x, y)} = \text{TPR}(\w)$. Here $\Ind{\cdot}$ denotes the indicator function. For the sake of convenience we will use $P(\w)=\EE{(\x, y)}{r^+(\w; \x, y)}$ and $N(\w)=\EE{(\x, y)}{r^-(\w; \x, y)}$ to denote population averages of the reward functions. We shall assume that our reward function $r$ is concave, $L_r$-Lipschitz, and takes values in a bounded range $[-B_r,B_r]$.\\

\noindent\textbf{Examples of Surrogate Reward Functions} Some examples of reward functions that are surrogates for the classification accuracy indicator function $\Ind{y\hat{y} > 0}$ are the inverted hinge loss function
\[
r_{\text{hinge}}(\hat y, y) = \max\bc{1,y\cdot\hat y}
\]
and the inverted logistic regression function
\[
r_{\text{logit}}(\hat y, y) = 1-\ln(1+\exp(-y\cdot\hat y))
\]
We will also experiment with non-surrogate (dubbed NS) versions of our algorithms which use TPR and TNR values directly. These will be discussed in Section~\ref{sec:exps}.


\subsection{Performance Measures}
\label{evaluationmeasures}
\noindent The task of quantification requires estimating the distribution of unlabeled items across a set $\mathcal{C}$ of available classes, with $|\mathcal{C}|=2$ in the binary setting. In our work we will target quantification performance measures as well as ``hybrid'' classification-quantification performance measures. We discuss them in turn. \\

\noindent\textbf{\kld: Kullback-Leibler Divergence}
In recent years this performance measure has become a standard in the quantification literature, in the evaluation of both binary and multiclass quantification \cite{Barranquero:2015fr,Esuli:2015gh,Gao:2015ly}. We redefine \kld below for convenience.
\begin{equation}
  \label{eq:KLD2}
  \kld(p,\hat{p}) = \sum_{c\in \mathcal{C}} 
  p(c)\log\frac{p(c)}{\hat{p}(c)}
\end{equation}
For distributions $p, p'$ over $\mathcal C$, the values $\kld(p,p')$ can range between 0 (perfect quantification) and $+\infty$.\footnote{$\kld$ is not a particularly well-behaved performance measure, since it is capable of taking unbounded values within the compact domain of the unit simplex. This poses a problem for optimization algorithms from the point of view of convergence, as well as numerical stability. To solve this problem, while computing $\kld$ for two distributions $p$ and $\hat p$, we can perform an \emph{additive smoothing} of both $p(c)$ and $\hat{p}(c)$ by computing
\begin{equation}
  \label{eq:smoothing}
  \nonumber p_{s}(c)=\frac{\epsilon+p(c)}{\epsilon|\mathcal{C}|+\displaystyle\sum_{c\in 
  \mathcal{C}}p(c)}
\end{equation}
\noindent where $p_{s}(c)$ denotes the smoothed version of $p(c)$. The denominator here is just a normalizing factor. The quantity $\epsilon=\frac{1}{2 |\mathcal{S}|}$ is often used as a smoothing factor, and is the one we adopt here. The smoothed versions of $p(c)$ and $\hat{p}(c)$ are then used in place of the non-smoothed versions in Equation \ref{eq:KLD}. We can show that, as a result, \kld is always bounded by $\kld(p_s,\hat p_s) \leq \O{\log\frac{1}{\epsilon}}$ However, we note that the smoothed \kld still returns a value of 0 when $p$ and $\hat{p}$ are identical distributions.}. Note that since \kld is a distance function and all our algorithms will be driven by reward maximization, for uniformity we will, instead of trying to minimize \kld, try to maximize $-\kld$; we will call this latter Neg\kld. \\

\noindent\textbf{NSS: Normalized Squared Score} This measure of quantification accuracy was introduced in \cite{Barranquero:2015fr}, and is defined as $\nss = 1 - (\frac{\fn-\fp}{\max\{p,(1-p)\}|\mathcal S|})^{2}$. Ignoring normalization constants, this performance measure attempts to reduce $\abs{\fn-\fp}$, a direct measure of quantification error.\\

We recall from Section \ref{intro} that several works have advocated the use of hybrid, ``multi-objective'' performance measures, that try to balance quantification and classification performance. These measures typically take a quantification performance measure such as \kld or \nss, and combine it with a classification performance measure. Typically, a classification performance measure that is sensitive to class imbalance \cite{NarasimhanKJ2015} is chosen, such as Balanced Accuracy $\ba=\frac{1}{2}(\tpr+\tnr)$ \cite{Barranquero:2015fr}, F-measure, or G-mean \cite{NarasimhanKJ2015}. Two such hybrid performance measures that are discussed in literature are presented below.\\

\noindent\textbf{CQB: Classification-Quantification Balancing} The work of \cite{Milli:2013fk} introduced this performance measure in an attempt to compromise between classification and quantification accuracy. As discussed in Section~\ref{intro}, this performance measure is defined as
\[
\text{CQB} = |\fp^{2}-\fn^{2}| = |\fp-\fn|\cdot(\fp+\fn),
\]
i.e. a product of $|\fn-\fp|$, a measure of quantification error, and $(\fn+\fp)$, a measure of classification error.\\

\noindent\textbf{QMeasure} The work of Barranquero et al.\ \cite{Barranquero:2015fr} introduced a generic scheme for constructing hybrid performance measures, using the so-called Q-measure defined as
\begin{equation}
  \label{eq:qmeasure}
  Q_{\beta}=(1+\beta^{2})\cdot\frac{\mathcal{P}_{\rm{class}}\cdot \mathcal{P}_{\rm{quant}}}{\beta^{2}\mathcal{P}_{\rm{class}} + \mathcal{P}_{\rm{quant}}},
  \end{equation}
that is, a weighted combination of a measure of classification accuracy $\mathcal{P}_{\rm{class}}$ and a measure of quantification accuracy $\mathcal{P}_{\rm{quant}}$. For the sake of simplicity, in our experiments we will adopt $\ba=\frac{1}{2}(\tpr+\tnr)$ as our $\mathcal{P}_{\rm{class}}$ and NSS as our $\mathcal{P}_{\rm{quant}}$. However, we stress that our methods can be suitably adapted to work with other choices of $\mathcal{P}_{\rm{class}}$ and $\mathcal{P}_{\rm{quant}}$.\\

We also introduce three new hybrid performance measures in this paper as a way of testing our optimization algorithms. We define these below and refer the reader to Tables~\ref{tab:nestedexamples} and \ref{tab:pscexamples} for details.\\

\noindent\textbf{\bakld} This hybrid performance measure takes a weighted average of BA and Neg\kld; i.e. $\bakld = C\cdot\ba + (1-C)\cdot(-\kld)$. This performance measure gives the user a strong handle on how much emphasis should be placed on quantification and how much on classification performance. We will use \bakld in our experiments to show that our methods offer an attractive tradeoff between the two.\\

We now define two hybrid performance measures that are constructed by taking the \emph{ratio} of a classification and a quantification performance measures. The aim of this exercise is to obtain performance measures that mimic the F-measure, which is also a pseudolinear performance measure \cite{NarasimhanKJ2015}. The ability of our methods to directly optimize such complex performance measures will be indicative of their utility in terms of the freedom they allow the user to design objectives in a data- and task-specific manner.\\

\noindent\textbf{CQReward and BKReward} These hybrid performance measures are defined as $\text{CQReward}=\frac{\ba}{2-\nss}$ and $\text{BKReward}=\frac{\ba}{1+\kld}$. Notice that both performance measures are optimized when the numerator i.e. \ba is large, and the denominator is small which translates to \nss being large for CQReward and \kld being small for BKReward. Clearly, both performance measures encourage good performance with respect to both classification and quantification and penalize a predictor which either neglects quantification to get better classification performance, or the other way round.\\

The past section has seen us introduce a wide variety of quantification and hybrid performance measures. Of these, the Neg\kld, \nss, and Q-measure were already prevalent in quantification literature and we introduced \bakld, CQReward and BKReward. As discussed before, the aim of exploring such a large variety of performance measures is to both demonstrate the utility of our methods with respect to the quantification problem, and present newer ways of designing hybrid performance measures that give the user more expressivity in tailoring the performance measure to the task at hand.

We also note that these performance measures have extremely diverse and complex structures. We can show that Neg\kld, Q-measure, and \bakld are nested concave functions, more specifically, concave functions of functions that are themselves concave in the confusion matrix of the predictor. On the other hand, CQReward and BKReward turn out to be pseudo-concave functions of the confusion matrix. Thus, we are working with two very different families of performance measures here, each of which has different properties and requires different optimization techniques.\\

In the following section, we introduce two novel methods to optimize these two families of performance measures.

\section{Stochastic Optimization Methods for Quantification}

The previous discussion in Sections \ref{intro} and \ref{relatedwork} clarifies two aspects of efforts in the quantification literature. Firstly, specific performance measures have been developed and adopted for evaluating quantification performance including KLD, NSS, Q-measure etc. Secondly, algorithms that \emph{directly} optimize these performance measures are desirable, as is evidenced by recent works \cite{Barranquero:2015fr, Esuli:2010fk, Esuli:2015gh, Milli:2013fk}.

The works mentioned above make use of tools from optimization literature to learn linear (e.g. \cite{Esuli:2015gh}) and non-linear (e.g. \cite{Milli:2013fk}) models to perform quantification. The state of the art efforts in this direction have adopted the structural SVM approach for optimizing these performance measures with great success \cite{Barranquero:2015fr, Esuli:2015gh}. However, this approach comes with severe drawbacks.

The structural SVM \cite{Joachims05}, although a significant tool that allows optimization of arbitrary performance measures, suffers from two key drawbacks. Firstly, the structural SVM surrogate is not necessarily a tight surrogate for all performance measures, something that has been demonstrated in past literature \cite{KarNJ2015,NarasimhanKJ2015}, which can lead to poor training. But more importantly, optimizing the structural SVM surrogate requires the use of expensive cutting plane methods which are known to scale poorly with the amount of training data, as well as are unable to handle streaming data.

To alleviate these problems, we propose stochastic optimization algorithms that \emph{directly} optimize a large family of quantification performance measures. Our methods come with sound theoretical convergence guarantees, are able to operate with streaming data sets and, as our experiments will demonstrate, offer much faster and accurate quantification performance on a variety of data sets.

Our optimization techniques introduce crucial advancements in the field of stochastic optimization of \emph{multivariate performance measures} and address the two families of performance measures discussed while concluding Section~\ref{formulation} -- 1) nested concave performance measures and 2) pseudo-concave performance measures. We describe these in turn below.

\subsection{Nested Concave Performance Measures}

\noindent The first class of performance measures that we deal with are concave combinations of concave performance measures. More formally, given three concave functions $\Psi, \zeta_1, \zeta_2: \R^2 \rightarrow \R$, we can define a performance measure
\[
\Pf_{(\Psi, \zeta_1, \zeta_2)}(\w) = \Psi(\zeta_1(\w),\zeta_2(\w)),
\]
where we have
\begin{align*}
\zeta_1(\w) &= \zeta_1(P(\w),N(\w))\\
\zeta_2(\w) &= \zeta_2(P(\w),N(\w)),
\end{align*}
where $P(\w)$ and $N(\w)$ can respectively denote, either the TPR and TNR values or surrogate reward functions therefor. Examples of such performance measures include the negative KLD performance measure and the QMeasure which are described in Section~\ref{evaluationmeasures}. Table~\ref{tab:nestedexamples} describes these performance measures in canonical form i.e. their expressions in terms of TPR and TNR values.

\begin{table*}[t]
\caption{A list of nested concave performance measures and their canonical expressions in terms of the confusion matrix $\Psi(P,N)$ where $P$ and $N$ denote the TPR, TNR values and $p$ and $n$ denote the proportion of positives and negatives in the population. The 4th, 6th and 8th columns give the closed form updates used in steps 15-17 in Algorithm~\ref{algo:nemsis}.}
\centering
\hspace*{-6ex}
\small{
\begin{tabular}{|c|c|c|c|c|c|c|c|}
\hline
Name & Expression & $\Psi(x,y)$ & $\vgamma(\qew)$ & $\zeta_1(P,N)$ & $\valpha(\rew)$ & $\zeta_2(P,N)$ & $\vbeta(\rew)$\\
\hline
NegKLD \cite{Barranquero:2015fr,Esuli:2015gh}& $\scriptstyle-\kld(p,\hat p)$ & $\scriptstyle p\cdot x + n\cdot y$ & $\scriptstyle (p,n)$ & $\scriptstyle\log(p P + n(1 - N))$ & $\scriptstyle \br{\frac{1}{r_1},\frac{1}{r_2}}$ & $\scriptstyle\log(n N + p(1 - P))$ & $\scriptstyle \br{\frac{1}{r_1},\frac{1}{r_2}} $\\\hline
QMeasure$_\beta$\cite{Barranquero:2015fr} & $\frac{(1 + \beta^2)\cdot\ba\cdot\nss}{\beta^2\cdot\ba + \nss}$ & $\scriptstyle\frac{(1 + \beta^2)\cdot x\cdot y}{\beta^2\cdot x + y}$ & $\substack{(1+\beta^2)\cdot\br{z^2,\frac{(1-z)^2}{\beta^2}}\\z = \frac{q_2}{\beta^2q_1 + q_2}}$ & $\scriptstyle\frac{P + N}{2}$ & $\scriptstyle(\frac{1}{2},\frac{1}{2})$ & $\scriptstyle1 - (p(1-P) - n(1-N))^2$ & $\scriptstyle\substack{2(z,-z)\\z = r_2 - r_1}$\\
\hline
\bakld & $\scriptstyle C\cdot\ba + (1-C)\cdot(-\kld)$ & $\scriptstyle C\cdot x + (1-C)\cdot y$ & $\scriptstyle (C,1-C)$ & $\scriptstyle\frac{P + N}{2}$ & $\scriptstyle(\frac{1}{2},\frac{1}{2})$ & $\scriptstyle-\kld(P,N)$ & see above \\
\hline
\end{tabular}
}
\label{tab:nestedexamples}
\end{table*}

Before describing our algorithm for nested concave measures, we recall the notion of concave Fenchel conjugate of concave functions. For any concave function $f: \R^2 \rightarrow \R$ and any $(u,v) \in \R^2$, the (concave) Fenchel conjugate of $f$ is defined as
\[
f^\ast(u,v) = \inf_{(x,y) \in \R^2} \bc{ux + vy - f(x,y)}.
\]
Clearly, $f^\ast$ is concave. Moreover, it follows from the concavity of $f$ that for any $(x,y) \in \R^2$,
\[
f(x,y) = \inf_{(u,v) \in \R^2} \bc{xu + yv  - f^\ast(u,v)}.
\]

Below we state the properties of strong concavity and smoothness. These will be crucial in our convergence analysis.

\begin{defn}[Strong Concavity and Smoothness]
A function $f: \R^d \rightarrow \R$ is said to be $\alpha$-strongly concave and $\gamma$-smooth if for all $\x,\y \in \R^d$, we have
\[
-\frac{\gamma}{2}\norm{\x-\y}_2^2 \leq f(\x) - f(\y) - \ip{\nabla f(\y)}{\x - \y} \leq -\frac{\alpha}{2}\norm{\x-\y}_2^2.
\]
\end{defn}

We will assume that the functions $\Psi, \zeta_1$, and $\zeta_2$ defining our performance measures are $\gamma$-smooth for some constant $\gamma > 0$. This is true of all functions, save the $\log$ function which is used in the definition of the \kld quantification measure. However, if we carry out the smoothing step pointed out in Section~\ref{evaluationmeasures} with some $\epsilon > 0$, then it can be shown that the \kld function does become $\O{\frac{1}{\epsilon^2}}$-smooth. An important property of smooth functions, that would be crucial in our analyses, is a close relationship between smooth and strongly convex functions

\begin{thm}[\cite{Zalinescu2002}]
\label{thm:scss}
A closed, concave function $f$ is $\beta$ smooth iff its (concave) Fenchel conjugate $f^\ast$ is $\frac{1}{\beta}$-strongly concave.
\end{thm}

We are now in a position to present our algorithm \nemsis for stochastic optimization of nested concave functions. Algorithm~\ref{algo:nemsis} gives an outline of the technique. We note that a direct application of traditional stochastic optimization techniques \cite{pegasos-paper} to such nested performance measures as those considered here is not possible as discussed before. \nemsis, overcomes these challenges by exploiting the nested dual structure of the performance measure by carefully balancing updates at the inner and outer levels.

At every time step, \nemsis performs four very cheap updates. The first update is a \emph{primal} ascent update to the model vector which takes a weighted stochastic gradient descent step. Note that this step involves a \emph{projection} step to the set of model vectors $\W$ denoted by $\Pi_\W(\cdot)$. In our experiments $\W$ was defined to be the set of all Euclidean norm-bounded vectors so that projection could be effected using Euclidean normalization which can be done in $\O{d}$ time if the model vectors are $d$-dimensional.

The weights of the descent step are decided by the dual parameters of the functions $\Psi, \zeta_1$, and $\zeta_2$. Then \nemsis updates the dual variables in three simple steps. In fact line numbers 15-17 can be executed in closed form (see Table~\ref{tab:nestedexamples}) for all the performance measures we see here which allows for very rapid updates. See Appendix~\ref{app:closed-form} for the simple derivations.

\begin{algorithm}[t]
	\caption{\small \nemsis: NEsted priMal-dual StochastIc updateS}
	\label{algo:nemsis}
	\begin{algorithmic}[1]
		\small{
			\REQUIRE Outer wrapper function $\Psi$, inner performance measures $\zeta_1, \zeta_2$, step sizes $\eta_t$, feasible sets $\W, \A$
			\ENSURE Classifier $\w \in \W$
			\STATE $\w_0 \< \vz, t \< 0, \bc{\rew_0 , \qew_0, \valpha_0, \vbeta_0, \vgamma_0} \< (0,0)$
			\WHILE{data stream has points}
				\STATE Receive data point $(\x_t,y_t)$
				\STATE \COMMENT{Perform primal ascent}
				\IF{$y_t > 0$}
					\STATE $\w_{t+1} \< \Pi_\W\br{\w_t + \eta_t(\gamma_{t,1}\alpha_{t,1} + \gamma_{t,2}\beta_{t,1})\nabla_\w r^+(\w_t; \x_t, y_t)}$
					\STATE $\qew_{t+1} \leftarrow t\cdot\qew_t + (\alpha_{t,1}, \beta_{t,1})\cdot r^+(\w_t; \x_t, y_t)$
				\ELSE
					\STATE $\w_{t+1} \< \Pi_\W\br{\w_t + \eta_t(\gamma_{t,1}\alpha_{t,2} + \gamma_{t,2}\beta_{t,2})\nabla_\w r^-(\w_t; \x_t, y_t)}$
					\STATE $\qew_{t+1} \leftarrow t\cdot\qew_t + (\alpha_{t,2}, \beta_{t,2})\cdot r^-(\w_t; \x_t, y_t)$
				\ENDIF
				\STATE $\rew_{t+1} \< (t+1)^{-1}\br{t\cdot\rew_t + (r^+(\w_t; \x_t, y_t), r^-(\w_t; \x_t, y_t))}$
				\STATE $\qew_{t+1} \< (t+1)^{-1}\br{\qew_{t+1} - (\zeta_1^\ast(\valpha_t),\zeta_2^\ast(\vbeta_t))}$
				\STATE \COMMENT{Perform dual updates}
				\STATE $\valpha_{t+1} = \underset{\valpha}{\arg\min}\bc{\valpha\cdot\rew_{t+1} - \zeta_1^\ast(\valpha)}$
				\STATE $\vbeta_{t+1} = \underset{\vbeta}{\arg\min}\bc{\vbeta\cdot\rew_{t+1} - \zeta_2^\ast(\vbeta)}$
				\STATE $\vgamma_{t+1} = \underset{\vgamma}{\arg\min}\bc{\vgamma\cdot\qew_{t+1} - \Psi^\ast(\vgamma)}$
				\STATE $t \< t+1$
			\ENDWHILE
			\STATE \textbf{return} $\barw = \frac{1}{t}\sum_{\tau = 1}^{t}\w_\tau$
		}
	\end{algorithmic}
\end{algorithm}

Below we state the convergence proof for \nemsis. We note that despite the complicated nature of the performance measures being tackled, \nemsis is still able to recover the optimal rate of convergence known for stochastic optimization routines. We refer the reader to Appendix~\ref{app:thm:nemsis-proof} for a proof of this theorem. The proof requires a careful analysis of the primal and dual update steps at different levels and tying the updates together by taking into account the nesting structure of the performance measure.

\begin{thm}
\label{THM:NEMSIS-ANALYSIS}
Suppose we are given a stream of random samples $(\x_1,y_1),\ldots,(\x_T,y_T)$ drawn from a distribution $\D$ over $\X\times\Y$. Let Algorithm~\ref{algo:nemsis} be executed with step sizes $\eta_t = \Theta(1/\sqrt t)$ with a nested concave performance measure $\Psi(\zeta_1(\cdot),\zeta_2(\cdot))$. Then, for some universal constant $C$, the average model $\barw = \frac{1}{T}\sum_{t=1}^T \w_t$ output by the algorithm satisfies, with probability at least $1 - \delta$,
\[
\Pf_{(\Psi, \zeta_1, \zeta_2)}(\barw) \geq \sup_{\w^\ast \in \W}\Pf_{(\Psi, \zeta_1, \zeta_2)}(\w^\ast) - C_{\Psi,\zeta_1,\zeta_2,r}\cdot\br{\frac{\log\frac{1}{\delta}}{\sqrt{T}}},
\]
where $C_{\Psi,\zeta_1,\zeta_2,r} = C(L_\Psi (L_r+B_r)(L_{\zeta_1}+L_{\zeta_2}))$ for a universal constant $C$ and $L_g$ denotes the Lipschitz constant of the function $g$.
\end{thm}

\textit{Related work of Narasimhan \etal}: Narasimhan \etal \cite{NarasimhanKJ2015} proposed an algorithm \pdsgd which offered stochastic optimization of concave performance measures. We note that although the performance measures considered here are indeed concave, it is difficult to apply \pdsgd to them directly since \pdsgd requires computation of gradients of the Fenchel dual of the function $\Pf_{(\Psi, \zeta_1, \zeta_2)}$ which are difficult to compute given the nested structure of this function. \nemsis, on the other hand, only requires the duals of the individual functions $\Psi,\zeta_1$, and $\zeta_2$ which are much more accessible. Moreover, \nemsis uses a much simpler dual update which does not involve any parameters and, in fact, has a closed form solution in all our cases. \pdsgd, on the other hand, performs dual gradient descent which requires a fine tuning of yet another step length parameter. A third benefit of \nemsis is that it achieves a logarithmic regret with respect to its dual updates (see the proof of Theorem~\ref{THM:NEMSIS-ANALYSIS}) whereas \pdsgd incurs a polynomial regret due to its gradient descent-style dual update.

\subsection{Pseudo-concave Performance Measures}

\noindent The next class of performance measures we consider can be expressed as a ratio of a quantification and a classification performance measure. More formally, given a \emph{convex} quantification performance measure $\Pf_\quant$ and a \emph{concave} classification performance measure $\Pf_\class$, we can define a performance measure
\[
\Pf_{(\Pf_\quant,\Pf_\class)}(\w) = \frac{\Pf_\class(\w)}{\Pf_\quant(\w)},
\]
We assume that both the performance measures, $\Pf_\quant$ and $\Pf_\class$, are positive valued. Such performance measures can be very useful in allowing a system designer to balance classification and quantification performance. Moreover, the form of the measure allows an enormous amount of freedom in choosing the quantification and classification performance measures. Examples of such performance measures include the CQReward and the BKReward measures. These were introduced in Section~\ref{evaluationmeasures} and are represented in their canonical forms in Table~\ref{tab:pscexamples}.

\begin{table}[t]
\caption{List of pseudo-concave performance measures and their canonical expressions in terms of the confusion matrix $\Psi(P,N)$. Note that $p$ and $n$ denote the proportion of positives and negatives in the population.}
\centering
\small{
\begin{tabular}{|c|c|c|c|}
\hline
Name & Expression & $\Pf_\quant(P,N)$ & $\Pf_\class(P,N)$\\
\hline
CQReward & $\frac{\ba}{2 - \nss}$ & $\scriptstyle1 + (p(1-P) - n(1-N))^2$ & $\frac{P + N}{2}$ \\
\hline
BKReward & $\frac{\ba}{1 + \kld}$ & \kld: see Table~\ref{tab:nestedexamples} & $\frac{P + N}{2}$ \\
\hline
\end{tabular}
}
\label{tab:pscexamples}
\end{table}

Performance measures, constructed the way described above, with a ratio of a concave over a convex measures, are called \emph{pseudo-concave} measures. This is because, although these functions are not concave, their level sets are still convex which makes it possible to optimize them efficiently. To see the intuition behind this, we need to introduce the notion of the \emph{valuation function} corresponding to the performance measure. As a passing note, we remark that because of the non-concavity of these performance measures, \nemsis cannot be applied here.

\begin{defn}[Valuation Function]
The valuation of a pseudo-concave performance measure $\Pf_{(\Pf_\quant,\Pf_\class)}(\w) = \frac{\Pf_\class(\w)}{\Pf_\quant(\w)}$ at any level $v > 0$, is defined as
\[
V(\w,v) = \Pf_\class(\w) - v\cdot\Pf_\quant(\w)
\]
\end{defn}

It can be seen that the valuation function defines the level sets of the performance measure. To see this, notice that due to the positivity of the functions $\Pf_\quant$ and $\Pf_\class$, we can have $\Pf_{(\Pf_\quant,\Pf_\class)}(\w) \geq v$ iff $V(\w,v) \geq 0$. However, since $\Pf_\class$ is concave, $\Pf_\quant$ is convex, and $v > 0$, $V(\w,v)$ is a concave function of $\w$.
 
This close connection between the level sets and notions of valuation functions have been exploited before to give optimization algorithms for pseudo-linear performance measures such as the F-measure \cite{NarasimhanKJ2015,ParambathUG14}. These approaches treat the valuation function as some form of proxy or surrogate for the original performance measure and optimize it in hopes of making progress with respect to the original measure.

Taking this approach with our performance measures yields a very natural algorithm for optimizing pseudo-concave measures which we outline in the \scanfc algorithm Algorithm~\ref{alg:scanfc}. \scanfc repeatedly trains models to optimize their valuations at the current level, then upgrades the level itself. Notice that step 4 in the algorithm is a concave maximization problem over a convex set, something that can be done using a variety of methods -- in the following we will see how \nemsis can be used to implement this step. Also notice that step 5 can, by the definition of the valuation function, be carried out by simply setting $v_{t+1} = \Pf_{(\Pf_\quant,\Pf_\class)}(\w_{t+1})$.

\begin{algorithm}[t]
	\caption{\small \scanfc: Concave AlternatioN}
	\label{alg:scanfc}
	\begin{algorithmic}[1]
		\small{
			\REQUIRE Objective $\Pf_{(\Pf_\quant,\Pf_\class)}$, model space $\W$, tolerance $\epsilon$
			\ENSURE An $\epsilon$-optimal classifier $\w \in \W$
			\STATE Construct the valuation function $V$
			\STATE $\w_0 \< \vz, t \< 1$
			\WHILE{$v_t > v_{t-1} + \epsilon$}
				\STATE $\w_{t+1} \< \mathop{\arg\max}_{\w\in\W} V(\w,v_t)$
				\STATE $v_{t+1} \< \mathop{\arg\max}_{v > 0} v$ such that $V(\w_{t+1},v) \geq v$
				\STATE $t \< t+1$
			\ENDWHILE
			\STATE \textbf{return} $\w_t$
		}
	\end{algorithmic}
\end{algorithm}

It turns out that \scanfc has a linear rate of convergence for well-behaved performance measures. The next result formalizes this statement. We note that this result is similar to the one arrived by \cite{NarasimhanKJ2015} but only for \emph{pseudo-linear} functions.

\begin{thm}
\label{THM:SCANFC-CONV}
Suppose we execute Algorithm~\ref{alg:scanfc} with a pseudo-concave performance measure $\Pf_{(\Pf_\quant,\Pf_\class)}$ such that the quantification performance measure always takes values in the range $[m,M]$, where $M > m > 0$.  Let $\Pf^\ast := \sup_{\w\in\W}\Pf_{(\Pf_\quant,\Pf_\class)}(\w)$ be the optimal performance level and $\Delta_t = \Pf^\ast - \Pf_{(\Pf_\quant,\Pf_\class)}(\w_t)$ be the excess error for the model $\w_t$ generated at time $t$. Then, for every $t > 0$, we have $\Delta_t \leq \Delta_0\cdot\br{1-\frac{m}{M}}^t$.
\end{thm}

We refer the reader to Appendix~\ref{app:scanfcconv} for a proof of this theorem. This theorem generalizes the result of \cite{NarasimhanKJ2015} to the more general case of pseudo-concave functions. Note that for the pseudo-concave functions defined in Table~\ref{tab:pscexamples}, care is taken to ensure that the quantification performance measure satisfies $m > 0$.

A drawback of \scanfc is that it cannot operate in streaming data settings and requires a concave optimization oracle. However, we notice that for the performance measures in Table~\ref{tab:pscexamples}, the valuation function is always at least a nested concave function. This motivates us to use \nemsis to solve the inner optimization problems in an online fashion. Combining this with an online technique to approximately execute step 5 of of the \scanfc and gives us the \scan algorithm, outlined in Algorithm~\ref{alg:scan}.

Thoerem~\ref{THM:SCAN-CONV} shows that \scan enjoys a convergence rate similar to that of \nemsis. Indeed, \scan is able to guarantee an $\epsilon$-approximate solution after witnessing $\softO{1/\epsilon^2}$ samples which is equivalent to a convergence rate of $\softO{1/\sqrt T}$. The proof of this result is obtained by showing that \scanfc is robust to approximate solutions to the inner optimization problems. We refer the reader to Appendix~\ref{app:scanconv} for a proof of this theorem.

\begin{algorithm}[t]
	\caption{\small \scan: Stochastic Concave AlternatioN}
	\label{alg:scan}
	\begin{algorithmic}[1]
		\small{
			\REQUIRE Objective $\Pf_{(\Pf_\quant,\Pf_\class)}$, model space $\W$, step sizes $\eta_t$, epoch lengths $s_e, s'_e$
			\ENSURE Classifier $\w \in \W$
			\STATE $v_0\<0, t\<0, e\<0,\w_0\<\vz$
			\REPEAT
				\STATE \COMMENT{Learning phase}
				\STATE $\widetilde\w\<\w_e$
				\WHILE{$t < s_e$}
					\STATE Receive sample $(\x,y)$
					\STATE \COMMENT{\nemsis update with $V(\cdot,v_e)$ at time $t$}
					\STATE $\w_{t+1} \leftarrow$ \nemsis$(V(\cdot,v_e), \w_t, (\x,y), t)$
					\STATE $t\<t+1$
				\ENDWHILE
				\STATE $t\<0, e\<e+1, \w_{e+1}\<\widetilde\w$
				\STATE \COMMENT{Level estimation phase}
				\STATE $v_+\<0,v_-\<0$
				\WHILE{$t < s'_e$}
					\STATE Receive sample $(\x,y)$
					\STATE $v_y \< v_y + r^y(\w_e;\x,y)$\hspace*{5ex}\COMMENT{Collect rewards}
					\STATE $t\<t+1$
				\ENDWHILE
				\STATE $t\<0, v_e\< \frac{\Pf_\class(v_+,v_-)}{\Pf_\quant(v_+,v_-)}$
			\UNTIL{stream is exhausted}
			\STATE \textbf{return} $\w_e$
		}
	\end{algorithmic}
\end{algorithm}

\begin{thm}
\label{THM:SCAN-CONV}
Suppose we execute Algorithm~\ref{alg:scan} with a pseudo-concave performance measure $\Pf_{(\Pf_\quant,\Pf_\class)}$ such that $\Pf_\quant$ always takes values in the range $[m,M]$ with $m > 0$, with epoch lengths $s_e, s'_e = \frac{C_{\Psi,\zeta_1,\zeta_2,r}}{m^2}\br{\frac{M}{M-m}}^{2e}$ following a geometric rate of increase, where the constant $C_{\Psi,\zeta_1,\zeta_2,r}$ is the effective constant for the \nemsis analysis (Theorem~\ref{THM:NEMSIS-ANALYSIS}) for the inner invocations of \nemsis in \scan. Also let the excess error for the model $\w_e$ generated after $e$ epochs be denoted by $\Delta_e = \Pf^\ast - \Pf_{(\Pf_\quant,\Pf_\class)}(\w_e)$. Then after $e = \O{\log\br{\frac{1}{\epsilon}\log^2\frac{1}{\epsilon}}}$ epochs, we can ensure with probability at least $1-\delta$ that $\Delta_e \leq \epsilon$. Moreover, the number of samples consumed till this point, ignoring universal constants, is at most ${\frac{C_{\Psi,\zeta_1,\zeta_2,r}^2}{\epsilon^2}\br{\log\log\frac{1}{\epsilon}+\log\frac{1}{\delta}}\log^4{\frac{1}{\epsilon}}}$.
\end{thm}

\textit{Related work of Narasimhan \etal}: Narasimhan \etal \cite{NarasimhanKJ2015} also proposed two algorithms \am and \amsgd which seek to optimize \emph{pseudo-linear} performance measures. However, neither those algorithms nor their analyses transfer directly to the pseudo-concave setting. This is because, by exploiting the pseudo-linearity of the performance measure, \am and \amsgd are able to convert their problem to a sequence of cost-weighted optimization problems which are very simple to solve. This convenience is absent here and as mentioned above, even after creation of the valuation function, \scan still has to solve a possibly nested concave minimization problem which it does by invoking the \nemsis procedure on this inner problem. The proof technique used in \cite{NarasimhanKJ2015} for analyzing \am also makes heavy use of pseudo-linearity. The convergence proof of \scanfc, on the other hand, is more general and yet guarantees a linear convergence rate. 

\begin{table}
\centering
{
\caption{Statistics of data sets used.}
\begin{tabular}{|c|c|c|c|}
\hline
\textbf{Data Set}		&\textbf{Data Points}	& \textbf{Features}	& \textbf{Positives}\\\hline
KDDCup08	& 102,294		& 117		& 0.61\%\\\hline
PPI			& 240,249		& 85		& 1.19\%\\\hline
CoverType	& 581,012		& 54		& 1.63\%\\\hline
ChemInfo	& 2142			& 55		& 2.33\%\\\hline
Letter		& 20,000		& 16		& 3.92\%\\\hline
IJCNN-1		& 141,691		& 22		& 9.57\%\\\hline
Adult		& 48,842		& 123		& 23.93\%\\\hline
Cod-RNA		& 488,565		& 8			& 33.3\%\\\hline
\end{tabular}
}
\label{tab:dataset-stats}
\end{table}

\section{Experimental Results}
\label{sec:exps}

We carried out an extensive set of experiments on diverse set of benchmark and real-world data to compare our proposed methods with other state-of-the-art approaches.\\[-7pt]

\textbf{Data sets}: We used the following benchmark data sets from the UCI machine learning repository : a) IJCNN, b) Covertype, c) Adult, d) Letters, and e) Cod-RNA. We also used the following three real-world data sets: a) Cheminformatics, a drug discovery data set from \cite{JorissenG2005}, b) 2008 KDD Cup challenge data set on breast cancer detection, and c) a data set pertaining to a protein-protein interaction (PPI) prediction task \cite{ppiQBK06}.  In each case, we used 70\% of the data for training and the remaining for testing.\\[-7pt]

\textbf{Methods}: We compares our proposed \nemsis and \scan algorithms\footnote{We will make code for our methods available publicly.} against the state-of-the-art one-pass mini-batch stochastic gradient method (\textbf{1PMB}) of \cite{KarNJ14}  and the \svmperf technique of \cite{JoachimsFY09}. Both these techniques are capable of optimizing structural SVM surrogates of arbitrary performance measures and we modified their implementations to suitably adapt them to the performance measures considered here. The \nemsis and \scan implementations used the hinge-based concave surrogate.\\[-7pt]

\textbf{Non-surrogate NS Approaches}: We also experimented with a variant of the \nemsis and \scan algorithms, where the dual updates were computed using original count based TPR and TNR values, rather than surrogate reward functions. We refer to this version as \nemsisns. We also developed a similar version of \scan called \scanns where the level estimation was performed using 0-1 TPR/TNR values. We empirically observed these non-surrogate versions of the algorithms to offer superior and more stable performance than the surrogate versions.\\

\begin{figure}[H]
\centering
\subfigure[Adult]{
\includegraphics[width=0.35\textwidth]{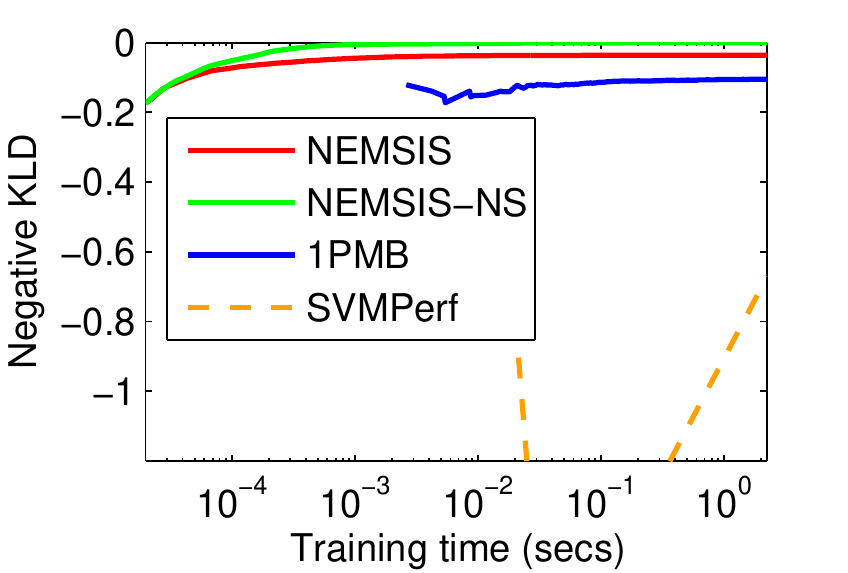}
\label{subfig:kld-a9a}
}
\subfigure[Cod-RNA]{
\includegraphics[width=0.35\textwidth]{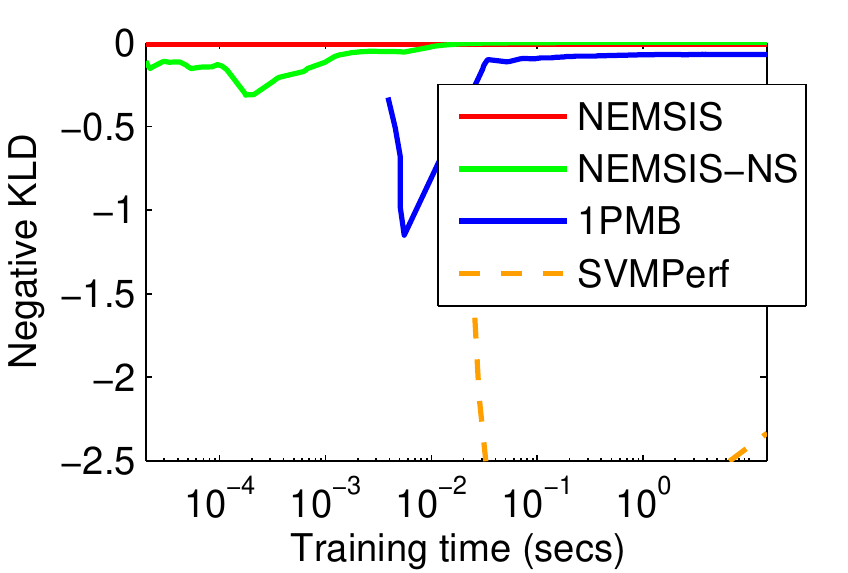}
\label{subfig:kld-cod-rna}
}\\
\subfigure[KDD08]{
\includegraphics[width=0.35\textwidth]{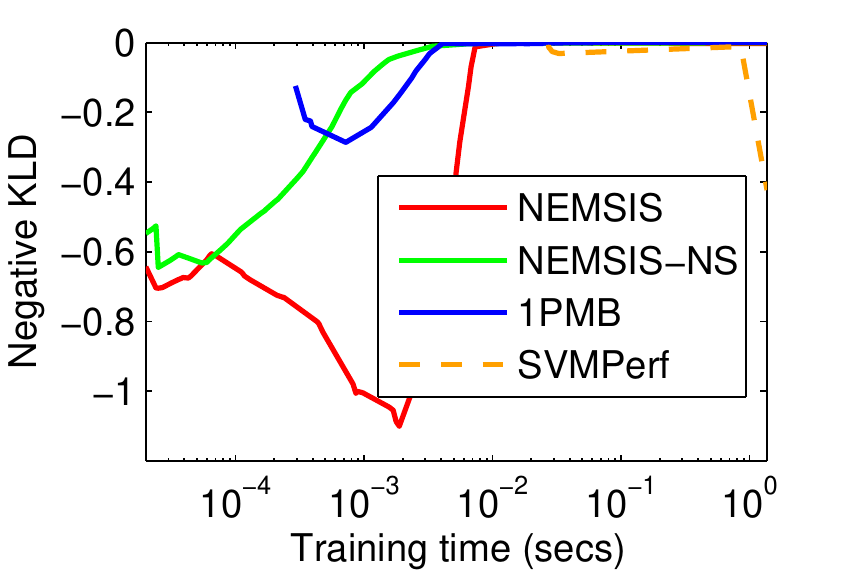}
\label{subfig:kld-kdd08}
}
\subfigure[PPI]{
\includegraphics[width=0.35\textwidth]{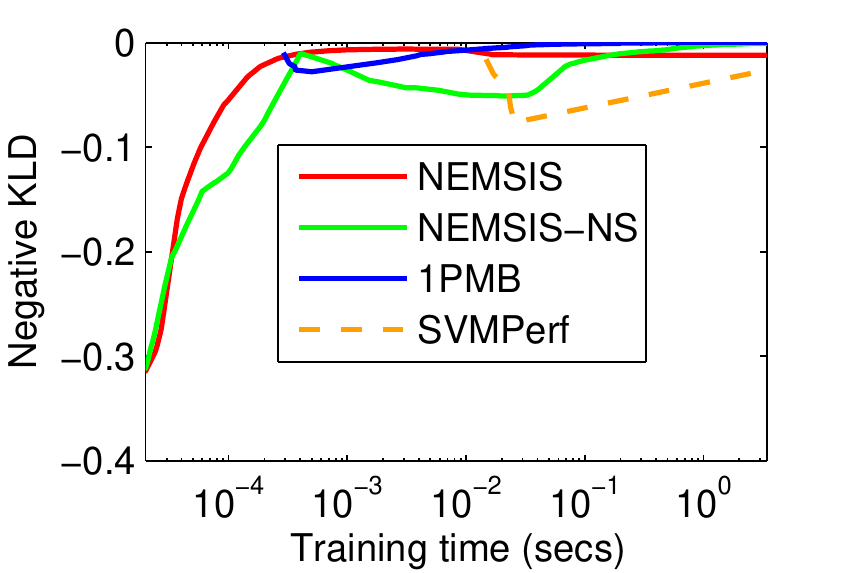}
\label{subfig:kld-ppi}
}
\caption{Experiments with \nemsis on Neg\kld: Plot of Neg\kld as a function of training time.}
\label{fig:KLD}
\end{figure}

\begin{figure}[H]
\centering
\subfigure[Adult]{
\includegraphics[width=0.35\textwidth]{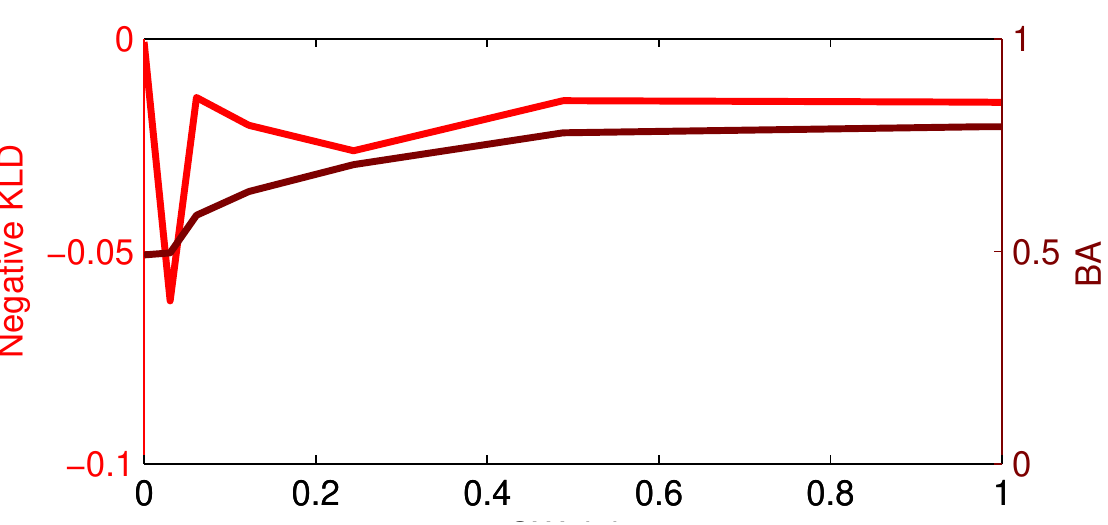}
\label{subfig:kldvsber-a9a}
}
\subfigure[Cod-RNA]{
\includegraphics[width=0.35\textwidth]{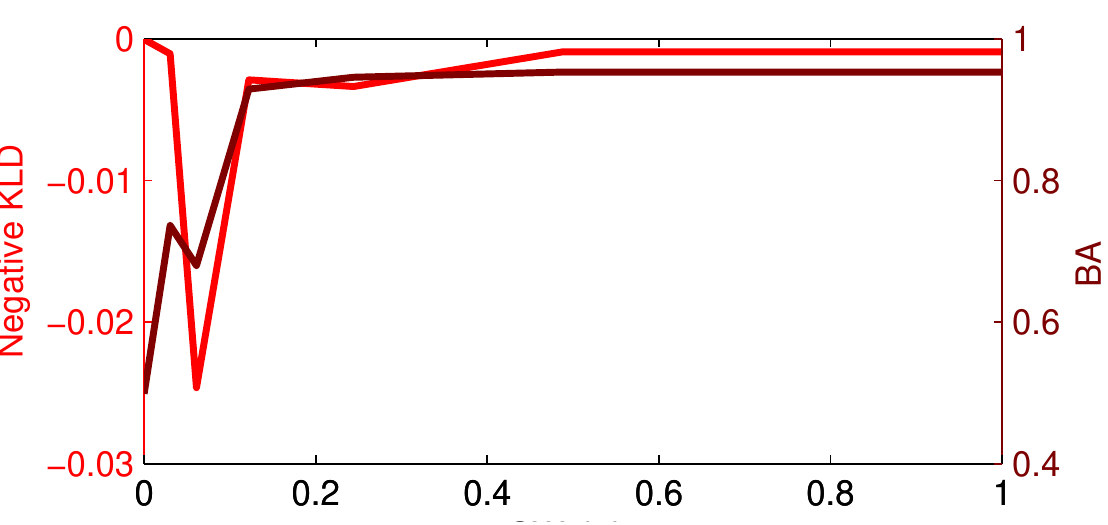}
\label{subfig:kldvsber-cod-rna}
}\\
\subfigure[Covtype]{
\includegraphics[width=0.35\textwidth]{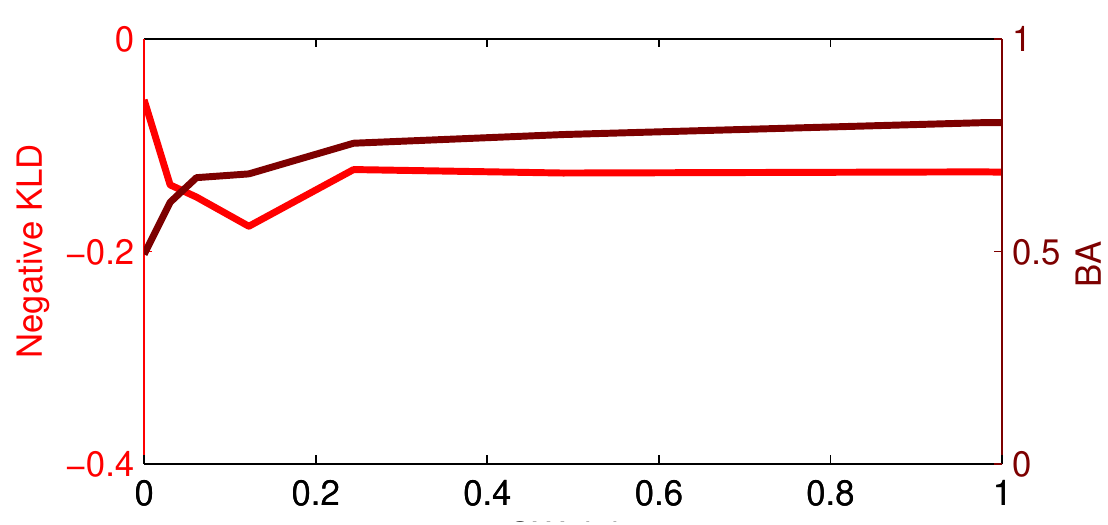}
\label{subfig:kldvsber-covtype}
}
\caption{Experiments on \nemsis with \bakld: Plots of quantification and classification performance as CWeight is varied.}
\label{fig:KLDvBER}
\end{figure}

\begin{figure}[H]
\centering
\includegraphics[width=0.25\textwidth]{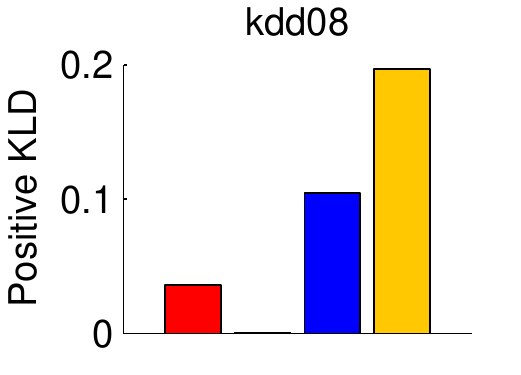}
\includegraphics[width=0.25\textwidth]{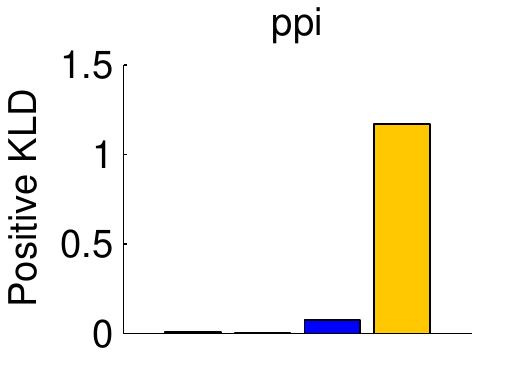}
\includegraphics[width=0.25\textwidth]{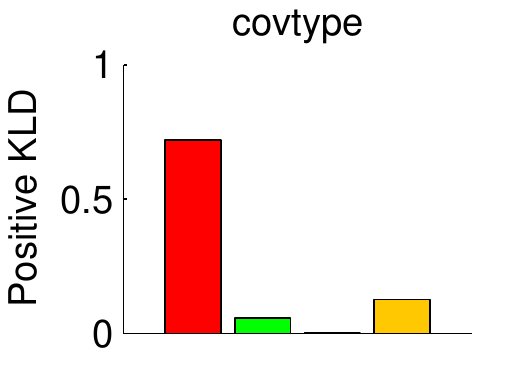}
\includegraphics[width=0.25\textwidth]{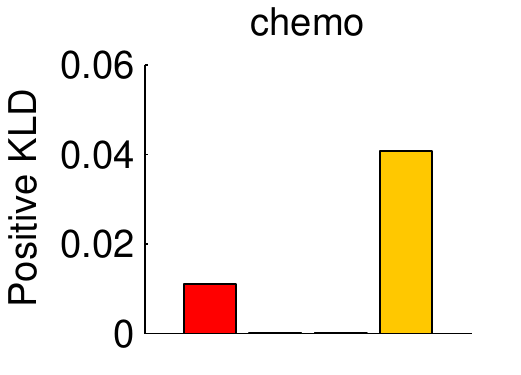}
\includegraphics[width=0.25\textwidth]{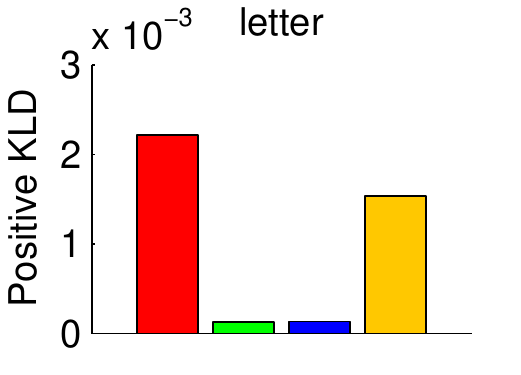}
\includegraphics[width=0.25\textwidth]{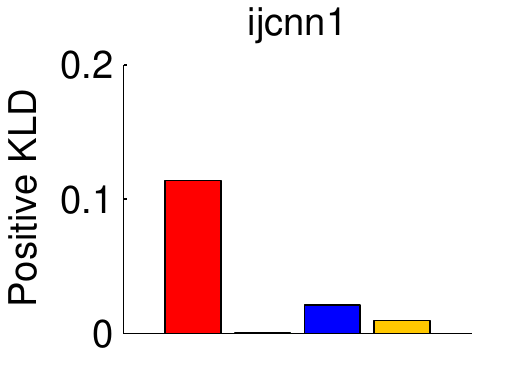}
\includegraphics[width=0.25\textwidth]{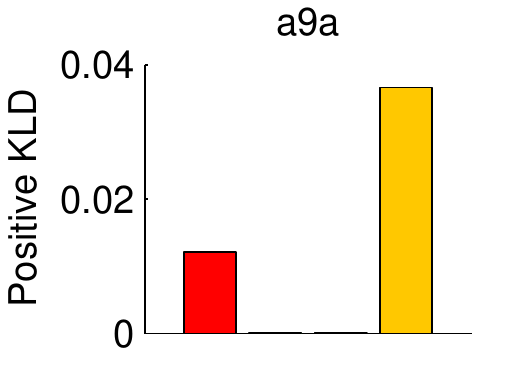}
\includegraphics[width=0.25\textwidth]{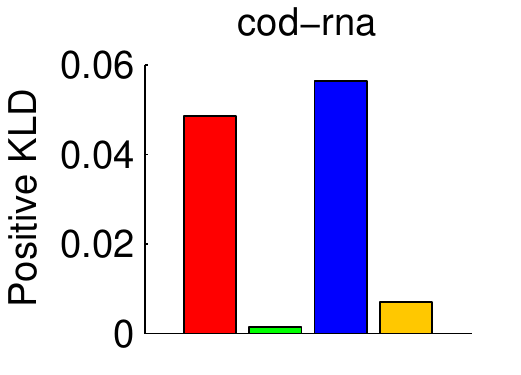}
\includegraphics[width=0.25\textwidth]{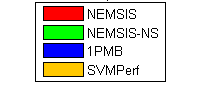}
\caption{A comparison of the \kld performance of various methods on data sets with varying class proportions (see Table~\ref{tab:dataset-stats}).}
\label{fig:prevalence}
\end{figure}

\begin{figure}[H]
\centering
\subfigure[Adult]{
\includegraphics[width=0.8\textwidth]{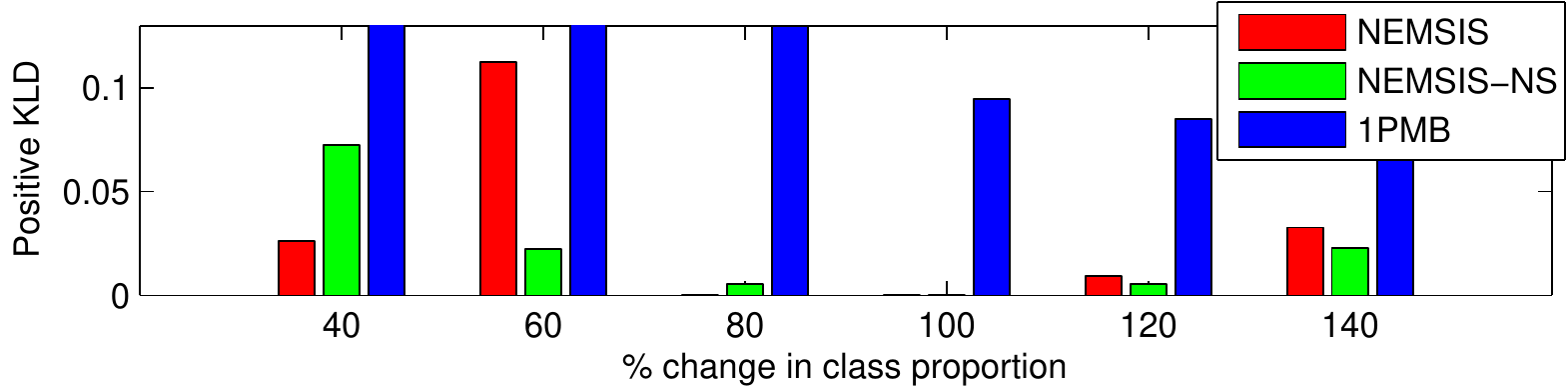}
\label{subfig:drift-a9a}
}
\subfigure[Letter]{
\includegraphics[width=0.8\textwidth]{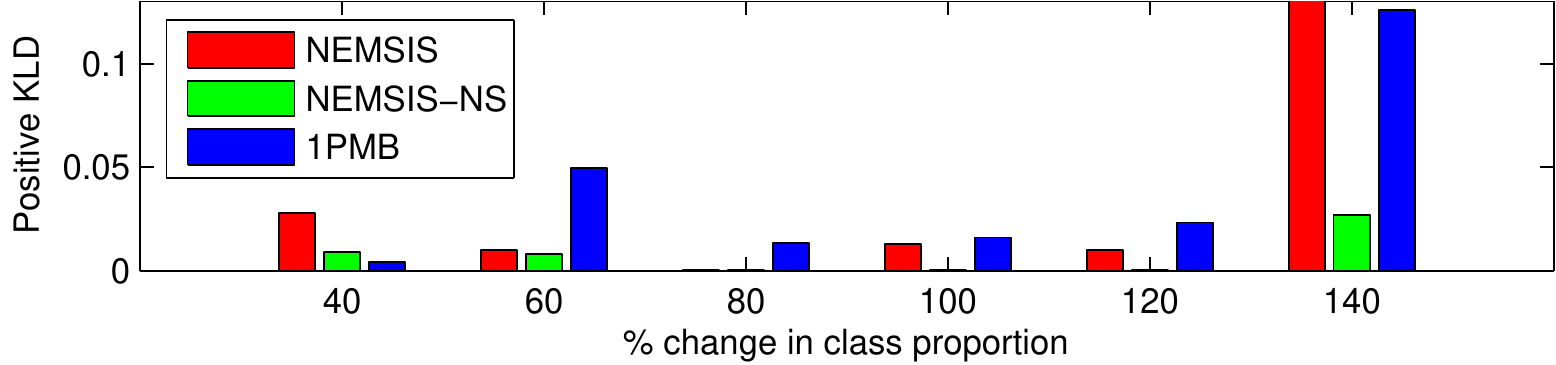}
\label{subfig:drift-letter}
}
\caption{A comparison of the \kld performance of various methods when distribution drift is introduced in the test sets.}
\label{fig:drift}
\end{figure}

\begin{figure}[H]
\centering
\subfigure[Adult]{
\includegraphics[width=0.35\textwidth]{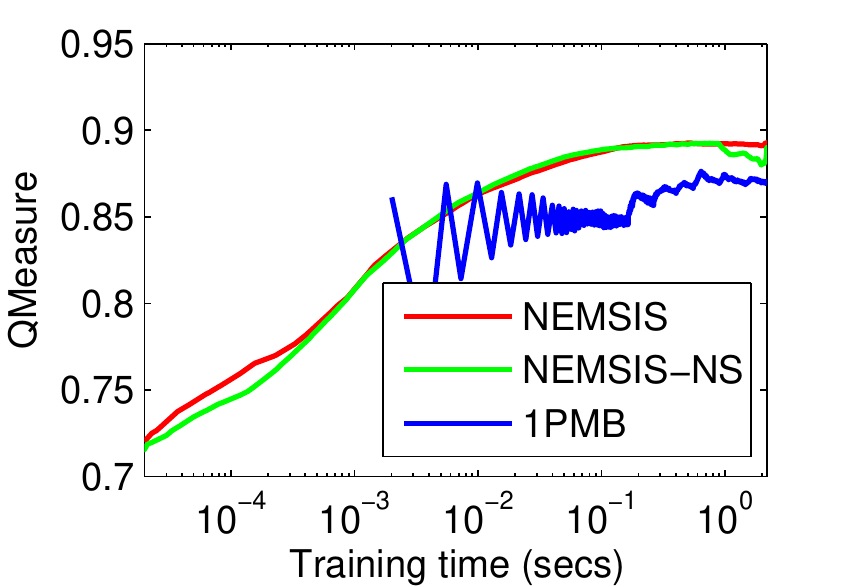}
\label{subfig:qmeasure-a9a}
}
\subfigure[Cod-RNA]{
\includegraphics[width=0.35\textwidth]{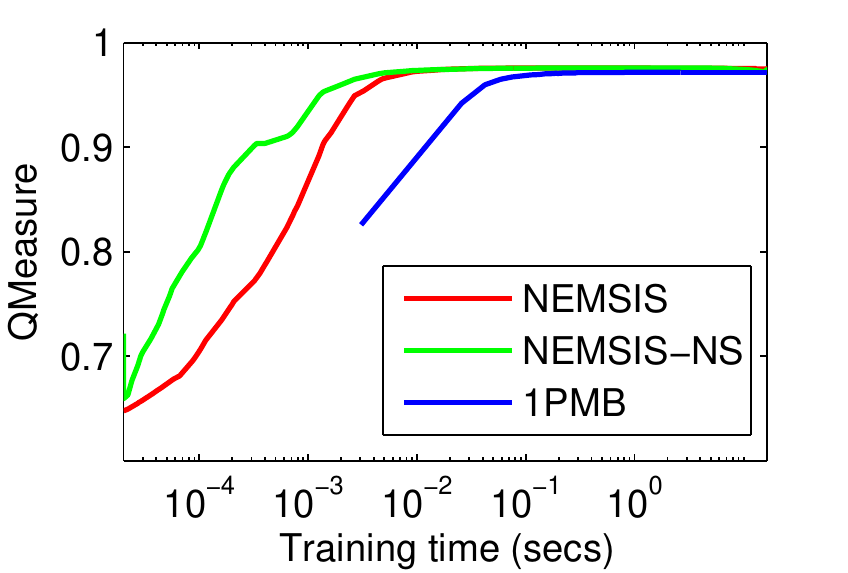}
\label{subfig:qmeasure-cod-rna}
}\\
\subfigure[IJCNN1]{
\includegraphics[width=0.35\textwidth]{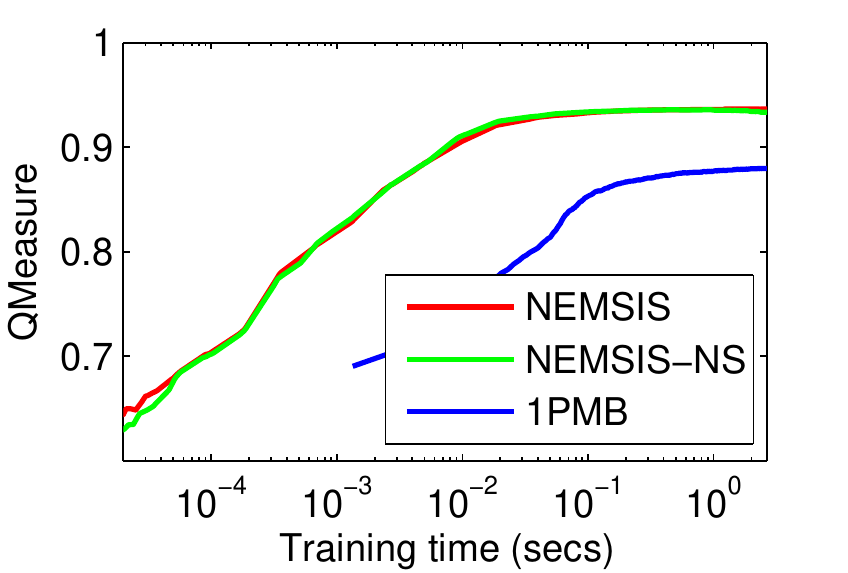}
\label{subfig:qmeasure-ijcnn1}
}
\subfigure[KDD08]{
\includegraphics[width=0.35\textwidth]{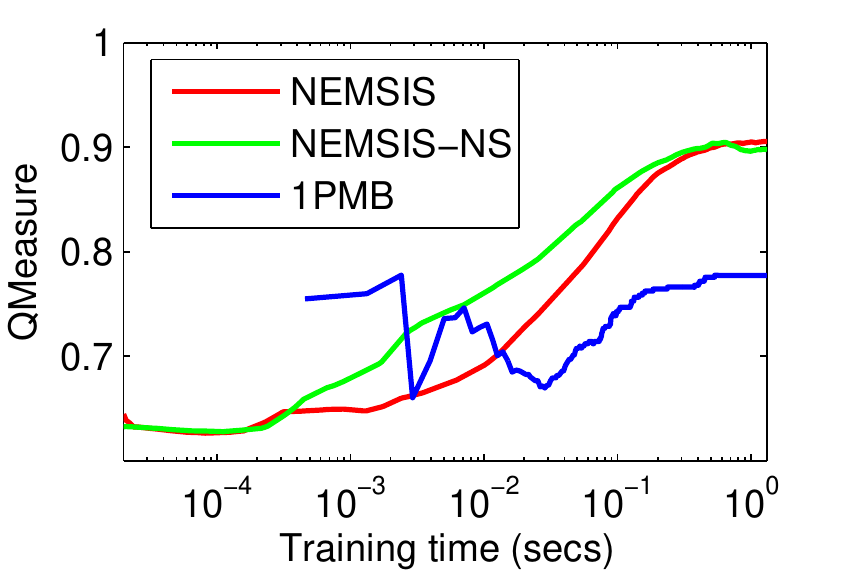}
\label{subfig:qmeasure-kdd08}
}
\caption{Experiments with \nemsis on Q-measure: Plot of Q-measure performance as a function of time.}
\label{fig:qmeasure}
\end{figure}

\begin{figure}[H]
\centering
\subfigure[Adult]{
\includegraphics[width=0.35\textwidth]{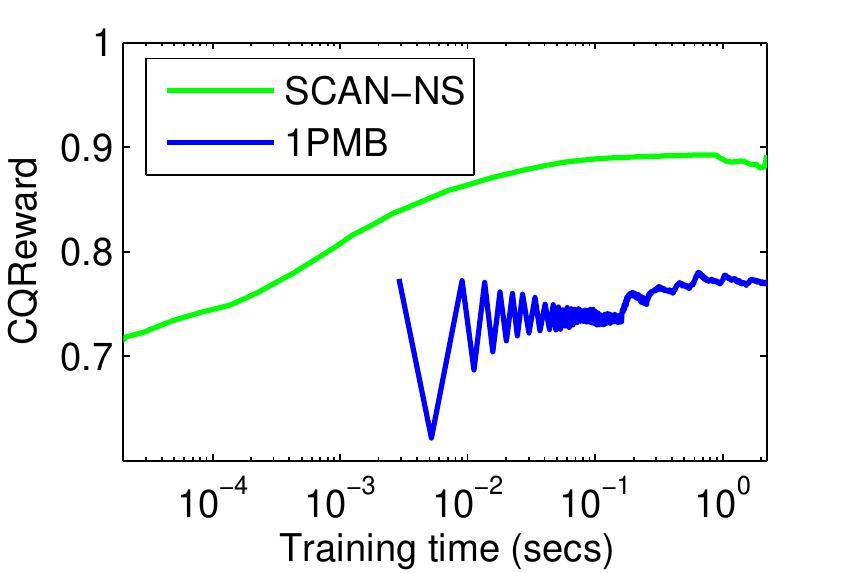}
\label{subfig:cqreward-a9a}
}
\subfigure[Cod-RNA]{
\includegraphics[width=0.35\textwidth]{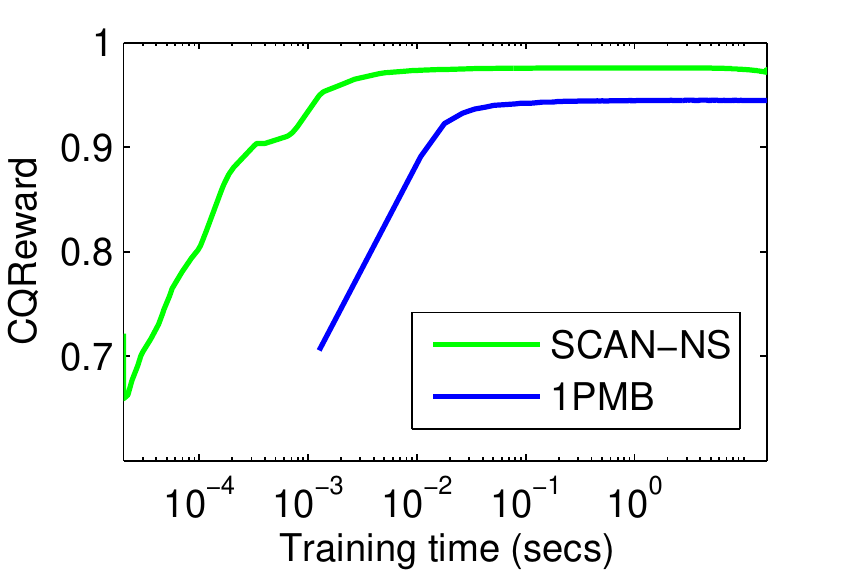}
\label{subfig:cqreward-cod-rna}
}\\
\subfigure[CovType]{
\includegraphics[width=0.35\textwidth]{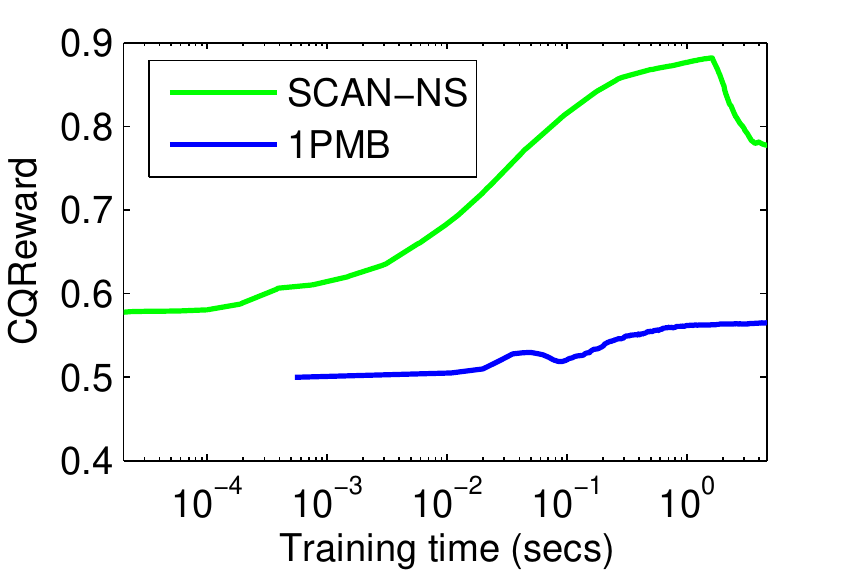}
\label{subfig:cqreward-covtype}
}
\subfigure[IJCNN1]{
\includegraphics[width=0.35\textwidth]{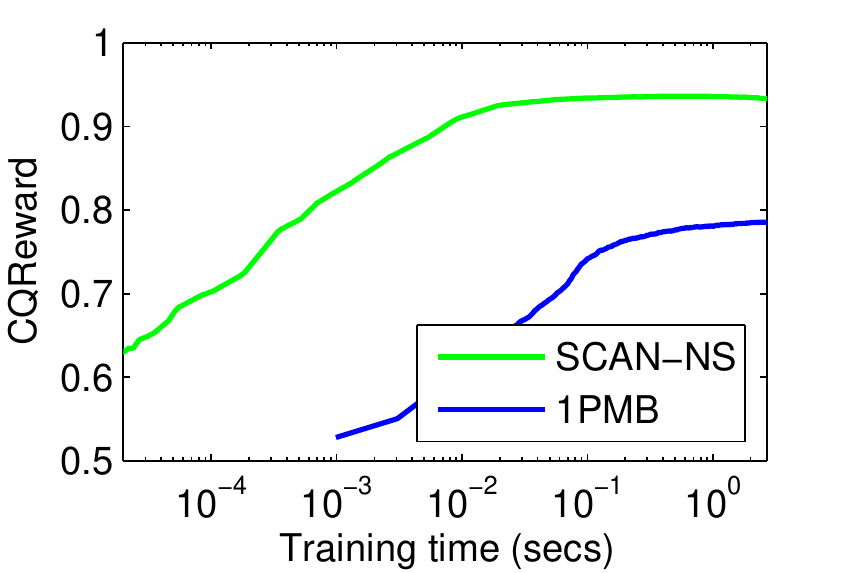}
\label{subfig:cqreward-ijcnn1}
}
\caption{Experiments with \scan on CQreward: Plot of CQreward performance as a function of time.}
\label{fig:cqreward}
\end{figure}

\textbf{Parameters}: All parameters including step sizes, upper bounds on reward functions, regularization parameters, and projection radii were tuned from the values $\{10^{-4},10^{-3},\ldots,10^3,10^4\}$ using a held-out portion of the training set treated as a validation set. For step sizes, the base step length $\eta_0$ was tuned from the above set and the step lengths were set to $\eta_t = \eta_0/\sqrt t$. In \textbf{1PMB}, we mimic the parameter setting in \cite{KarNJ14}, setting the buffer size to 500 and the number of passes to 25.\\

\textbf{Comparison on Neg\kld:} We first compare \nemsisns and \nemsis against the baselines \textbf{1PMB}  and \svmperf on several data sets on the negative KLD measure. The results are presented in Figure \ref{fig:KLD}. It is clear that the proposed algorithms have comparable performance with significantly
faster rate of convergence. Since \svmperf is a batch/off-line method, it is important to clarify how it was compared against the other online methods. In this case,
timers were embedded inside the \svmperf code, and at regular intervals, the performance of the current  model vector was evaluated. It is clear that \svmperf is significantly slower and its behavior is quite erratic. The proposed methods are often faster than \textbf{1PMB}. On three of the four data sets \nemsisns achieves a faster rate of convergence compared to \nemsis.\\

\textbf{Comparison on \bakld:}
We also used the \bakld performance measure to evaluate the trade-off \nemsis offers between quantification and classification performance. The weighting parameter $C$ in \bakld (see Table~\ref{tab:nestedexamples}), denoted here by CWeight to avoid confusion, was varied from 0 to 1 across a fine grid; for each value, \nemsis was used to optimize \bakld and its performance on \ba and \kld were noted separately. In the results presented in Figure \ref{fig:KLDvBER} for three data sets, notice that there is a sweet spot where the two tasks, i.e. quantification and classification simultaneously have good performance.\\

\textbf{Comparison under varying class proportions:}
We next evaluated the robustness of the algorithms across data sets with varying different class proportions (see Table~\ref{tab:dataset-stats} for the dataset label proportions). In Figure \ref{fig:prevalence}, we plot positive KLD (smaller values are better) for the proposed and baseline methods for these diverse datasets. Again, it is clear that the \nemsis family of algorithms of has better KLD performance compared to the baselines,
demonstrating their versatility across a range of class distributions.\\

\textbf{Comparison under varying drift:}
Next, we test the performance of the \nemsis family of methods when there are drifts in class proportions between the train and test sets. In each case, we retain the original class proportion in the train set, and vary the class proportions in the test set, by suitably sampling from the original set of positive and negative test instances.\footnote{More formally, we consider a setting where both the train and test sets are generated using the same conditional class distribution $\P(Y = 1\,|X)$, but with different marginal distributions over instances $\P(X)$, and thus, have different class proportions. Further, in these experiments, we made a simplistic assumption that there is no label noise; hence for any instance $\x$, $\P(Y=1\,|\,X=\x) = 1$ or $0$. Thus, we generated our test set with class proportion $p'$ by simply setting $\P(X = \x)$ to the following distribution: with probability $p'$, sample a point uniformly from all points with $\P(Y=1\,|\,X=\x) = 1$, and with probability $1-p'$, sample a point uniformly from all points with $\P(Y=1\,|\,X=\x) = 0$. 
}
 We have not included \svmperf in these experiments as it took an inordinately long time to complete. As seen in Figure \ref{fig:drift}, on the Adult and Letter data set the \nemsis family is fairly robust to small class drifts. As expected, when the class proportions change by a large amount  in the test set (over 100 percent), all algorithms perform poorly.\\

\textbf{Comparison on hybrid performance measures:}
Finally, we tested our methods in optimizing composite performance measures that strike a more nuanced trade-off between quantification and classification performance. Figures \ref{fig:qmeasure} contains results for the \nemsis methods while optimizing Q-measure (see Table \ref{tab:nestedexamples}), and Figure \ref{fig:cqreward} contains results for \scanns while optimizing CQReward (see Table \ref{tab:pscexamples}). The proposed methods are often significantly better than the baseline \textbf{1PMB} in terms of both accuracy and running time.

\section{Conclusion}

\noindent Quantification, the task of estimating class prevalence in problem settings subject to distribution drift, has emerged as an important problem in machine learning and data mining. Our discussion justified the necessity to design algorithms that exclusively solve the quantification task, with a special emphasis on performance measures such as the Kullback-Leibler divergence that is considered a de facto standard in the literature.

In this paper we proposed a family of algorithms \nemsis, \scanfc, \scan, and their non-surrogate versions, to address the online quantification problem. By abstracting Neg\kld and other hybrid performance measures as nested concave or pseudo concave functions we designed provably correct and efficient algorithms for optimizing these performance measures in an online stochastic setting.

We validated our algorithms on several data sets under varying conditions, including class imbalance and distribution drift. The proposed algorithms demonstrate the ability to jointly optimize both quantification and classification tasks. To the best of our knowledge this is the first work which directly addresses the online quantification problem and as such, opens up novel application areas.

\section*{Acknowledgments}
The authors thank the anonymous reviewers for their helpful comments. PK thanks the Deep Singh and Daljeet Kaur Faculty Fellowship, and the Research-I Foundation at IIT Kanpur for support. SL acknowledges the support from 7Pixel S.r.l., SyChip Inc., and Murata Japan.

\bibliographystyle{plain}
\bibliography{Fabrizio,refs-icml15-tale-of-two-classes,newrefs}

\begin{thebibliography}{10}

\bibitem{Alaiz-Rodriguez:2011fk}
Roc{\'\i}o Ala{\'\i}z-Rodr{\'\i}guez, Alicia Guerrero-Curieses, and Jes{\'u}s
  Cid-Sueiro.
\newblock Class and subclass probability re-estimation to adapt a classifier in
  the presence of concept drift.
\newblock {\em Neurocomputing}, 74(16):2614--2623, 2011.

\bibitem{Balikas:2015mz}
Georgios Balikas, Ioannis Partalas, Eric Gaussier, Rohit Babbar, and
  Massih-Reza Amini.
\newblock Efficient model selection for regularized classification by
  exploiting unlabeled data.
\newblock In {\em Proceedings of the 14th International Symposium on
  Intelligent Data Analysis (IDA 2015)}, pages 25--36, Saint Etienne, {FR},
  2015.

\bibitem{Barranquero:2015fr}
Jos{\'e} Barranquero, Jorge D{\'\i}ez, and Juan~Jos{\'e} del Coz.
\newblock Quantification-oriented learning based on reliable classifiers.
\newblock {\em Pattern Recognition}, 48(2):591--604, 2015.

\bibitem{Bella:2010kx}
Antonio Bella, C{\`e}sar Ferri, Jos{\'e} Hern{\'a}ndez-Orallo, and
  Mar{\'\i}a~Jos{\'e} Ram{\'\i}rez-Quintana.
\newblock Quantification via probability estimators.
\newblock In {\em Proceedings of the 11th IEEE International Conference on Data
  Mining (ICDM 2010)}, pages 737--742, Sydney, {AU}, 2010.

\bibitem{online-batch-single}
Nicol\'o Cesa-Bianchi, Alex Conconi, and Claudio Gentile.
\newblock On the generalization ability of on-line learning algorithms.
\newblock In {\em Proceedings of the 15th Annual Conference on Neural
  Information Processing Systems (NIPS 2001)}, pages 359--366, Vancouver,
  {USA}, 2001.

\bibitem{Chan2006}
Yee~Seng Chan and Hwee~Tou Ng.
\newblock Estimating class priors in domain adaptation for word sense
  disambiguation.
\newblock In {\em Proceedings of the 44th Annual Meeting of the Association for
  Computational Linguistics (ACL 2006)}, pages 89--96, Sydney, {AU}, 2006.

\bibitem{Csiszar:2004fk}
Imre Csisz{\'a}r and Paul~C. Shields.
\newblock Information theory and statistics: {A} tutorial.
\newblock {\em Foundations and Trends in Communications and Information
  Theory}, 1(4):417--528, 2004.

\bibitem{du-Plessis:2012nr}
Marthinus~C. {du Plessis} and Masashi Sugiyama.
\newblock Semi-supervised learning of class balance under class-prior change by
  distribution matching.
\newblock In {\em Proceedings of the 29th International Conference on Machine
  Learning (ICML 2012)}, Edinburgh, UK, 2012.

\bibitem{Esuli:2010fk}
Andrea Esuli and Fabrizio Sebastiani.
\newblock Sentiment quantification.
\newblock {\em IEEE Intelligent Systems}, 25(4):72--75, 2010.

\bibitem{Esuli:2015gh}
Andrea Esuli and Fabrizio Sebastiani.
\newblock Optimizing text quantifiers for multivariate loss functions.
\newblock {\em ACM Transactions on Knowledge Discovery and Data}, 9(4):Article
  27, 2015.

\bibitem{Forman:2008kx}
George Forman.
\newblock Quantifying counts and costs via classification.
\newblock {\em Data Mining and Knowledge Discovery}, 17(2):164--206, 2008.

\bibitem{Gao:2015ly}
Wei Gao and Fabrizio Sebastiani.
\newblock Tweet sentiment: From classification to quantification.
\newblock In {\em Proceedings of the 7th International Conference on Advances
  in Social Network Analysis and Mining (ASONAM 2015)}, pages 97--104, Paris,
  {FR}, 2015.

\bibitem{shuai14icml}
Claudio Gentile, Shuai Li, and Giovanni Zappella.
\newblock Online clustering of bandits.
\newblock In {\em Proceedings of the 31st International Conference on Machine
  Learning (ICML 2014)}, Bejing, {CN}, 2014.

\bibitem{Gonzalez-Castro:2013fk}
V{\'\i}ctor Gonz{\'a}lez-Castro, Roc{\'\i}o Alaiz-Rodr{\'\i}guez, and Enrique
  Alegre.
\newblock Class distribution estimation based on the {H}ellinger distance.
\newblock {\em Information Sciences}, 218:146--164, 2013.

\bibitem{log-regret}
Elad Hazan, Adam Kalai, Satyen Kale, and Amit Agarwal.
\newblock {Logarithmic Regret Algorithms for Online Convex Optimization}.
\newblock In {\em Proceedings of the 19th Annual Conference on Learning Theory
  (COLT 2006)}, pages 499--513, Pittsburgh, {USA}, 2006.

\bibitem{Hopkins:2010fk}
Daniel~J. Hopkins and Gary King.
\newblock A method of automated nonparametric content analysis for social
  science.
\newblock {\em American Journal of Political Science}, 54(1):229--247, 2010.

\bibitem{Joachims05}
Thorsten Joachims.
\newblock A support vector method for multivariate performance measures.
\newblock In {\em Proceedings of the 22nd International Conference on Machine
  Learning (ICML 2005)}, pages 377--384, Bonn, {DE}, 2005.

\bibitem{JoachimsFY09}
Thorsten Joachims, Thomas Finley, and Chun{-}Nam~John Yu.
\newblock {Cutting-plane training of structural SVMs}.
\newblock {\em Machine Learning}, 77(1):27--59, 2009.

\bibitem{JorissenG2005}
Robert~N. Jorissen and Michael~K. Gilson.
\newblock Virtual screening of molecular databases using a support vector
  machine.
\newblock {\em Jounal of Chemical Information Modelling}, 45(3):549--561, 2005.

\bibitem{KarNJ14}
Purushottam Kar, Harikrishna Narasimhan, and Prateek Jain.
\newblock Online and stochastic gradient methods for non-decomposable loss
  functions.
\newblock In {\em Proceedings of the 28th Annual Conference on Neural
  Information Processing Systems (NIPS 2014)}, pages 694--702, Montreal, {USA},
  2014.

\bibitem{KarNJ2015}
Purushottam Kar, Harikrishna Narasimhan, and Prateek Jain.
\newblock Surrogate functions for maximizing precision at the top.
\newblock In {\em Proceedings of the 32nd International Conference on Machine
  Learning (ICML 2015)}, pages 189--198, Lille, {FR}, 2015.

\bibitem{King:2008fk}
Gary King and Ying Lu.
\newblock Verbal autopsy methods with multiple causes of death.
\newblock {\em Statistical Science}, 23(1):78--91, 2008.

\bibitem{Lewis95}
David~D. Lewis.
\newblock Evaluating and optimizing autonomous text classification systems.
\newblock In {\em Proceedings of the 18th ACM International Conference on
  Research and Development in Information Retrieval (SIGIR 1995)}, pages
  246--254, Seattle, {USA}, 1995.

\bibitem{Milli:2013fk}
Letizia Milli, Anna Monreale, Giulio Rossetti, Fosca Giannotti, Dino Pedreschi,
  and Fabrizio Sebastiani.
\newblock Quantification trees.
\newblock In {\em Proceedings of the 13th IEEE International Conference on Data
  Mining (ICDM 2013)}, pages 528--536, Dallas, {USA}, 2013.

\bibitem{NarasimhanKJ2015}
Harikrishna Narasimhan, Purushottam Kar, and Prateek Jain.
\newblock Optimizing non-decomposable performance measures: {A} tale of two
  classes.
\newblock In {\em proceedings of the 32nd International Conference on Machine
  Learning (ICML 2015)}, pages 199--208, Lille, {FR}, 2015.

\bibitem{Pan:2012fk}
Weike Pan, Erheng Zhong, and Qiang Yang.
\newblock Transfer learning for text mining.
\newblock In Charu~C. Aggarwal and ChengXiang Zhai, editors, {\em Mining Text
  Data}, pages 223--258. Springer, Heidelberg, {DE}, 2012.

\bibitem{ParambathUG14}
Shameem~P. Parambath, Nicolas Usunier, and Yves Grandvalet.
\newblock Optimizing {F-Measures} by cost-sensitive classification.
\newblock In {\em Proceedings of the 28th Annual Conference on Neural
  Information Processing Systems (NIPS 2014)}, pages 2123--2131, Montreal,
  {USA}, 2014.

\bibitem{ppiQBK06}
Yanjun Qi, Ziv Bar-Joseph, and Judith Klein-Seetharaman.
\newblock Evaluation of different biological data and computational
  classification methods for use in protein interaction prediction.
\newblock {\em Proteins}, 63:490--500, 2006.

\bibitem{Saerens:2002uq}
Marco Saerens, Patrice Latinne, and Christine Decaestecker.
\newblock Adjusting the outputs of a classifier to new a priori probabilities:
  A simple procedure.
\newblock {\em Neural Computation}, 14(1):21--41, 2002.

\bibitem{Shalev-ShwartzSSS2009}
Shai Shalev-Shwartz, Ohad Shamir, Nathan Srebro, and Karthik Sridharan.
\newblock {Stochastic Convex Optimization}.
\newblock In {\em Proceedings of the 22nd Annual Conference on Learning Theory
  (COLT 2009)}, Montreal, CA, 2009.

\bibitem{pegasos-paper}
Shai Shalev-Shwartz, Yoram Singer, Nathan Srebro, and Andrew Cotter.
\newblock {PEGASOS: Primal Estimated sub-GrAdient SOlver for SVM}.
\newblock {\em Mathematical Programming, Series B}, 127(1):3--30, 2011.

\bibitem{Xue:2009uq}
Jack~Chongjie Xue and Gary~M. Weiss.
\newblock Quantification and semi-supervised classification methods for
  handling changes in class distribution.
\newblock In {\em Proceedings of the 15th ACM International Conference on
  Knowledge Discovery and Data Mining (SIGKDD 2009)}, pages 897--906, Paris,
  {FR}, 2009.

\bibitem{Zalinescu2002}
Constantin Zalinescu.
\newblock {\em Convex Analysis in General Vector Spaces}.
\newblock River Edge, NJ: World Scientific Publishing, 2002.

\bibitem{Zhang:2010kx}
Zhihao Zhang and Jie Zhou.
\newblock Transfer estimation of evolving class priors in data stream
  classification.
\newblock {\em Pattern Recognition}, 43(9):3151--3161, 2010.

\bibitem{zinkevich}
Martin Zinkevich.
\newblock {Online Convex Programming and Generalized Infinitesimal Gradient
  Ascent}.
\newblock In {\em 20th International Conference on Machine Learning (ICML)},
  pages 928--936, 2003.

\end{thebibliography}

\appendix

\section{Deriving Updates for \large{\nemsis}}
\label{app:closed-form}

The derivation of the closed form updates for steps 15-17 in the \nemsis algorithm (see Algorithm~\ref{algo:nemsis}) starts with the observation that in all the nested concave performance measures considered here, the outer and the inner concave functions, namely $\Psi, \zeta_1, \zeta_2$ are concave, continuous, and differentiable. The logarithm function is non-differentiable at 0 but the smoothing step (see Section~ref{formulation}) ensures that we will never approach 0 in our analyses or the execution of the algorithm. The derivations hinge on the following basic result from convex analysis \cite{Zalinescu2002}:
\begin{lem}
\label{lem:mirror}
Let $f$ be a closed, differentiable and concave function and $f^\ast$ be its concave Fenchel dual. Then $\nabla f^\ast = (\nabla f)^{-1}$ i.e. for any $\x \in \X$ and $\vec u \in \X^\ast$ (the space of all linear functionals over $\X$), we have $\nabla f^\ast(\vec u) = \x$ iff $\nabla f(\x) = \vec u$.
\end{lem}

Using this result, we show how to derive the updates for $\vgamma$. The updates for $\vbeta$ and $\valpha$ follow similarly. We have
\[
\vgamma_t = \underset{\vgamma}{\arg\min}\bc{\vgamma\cdot\qew_t - \Psi^\ast(\vgamma)}
\]
By first order optimality conditions, we get that $\vgamma_t$ can minimize the function $h(\gamma) = \vgamma\cdot\qew_t - \Psi^\ast(\vgamma)$ only if $\qew_t = \nabla\Psi^\ast(\vgamma_t)$. Using Lemma~\ref{lem:mirror}, we get $\vgamma_t = \nabla\Psi(\qew_t)$. Using this technique, all the closed form expressions can be readily derived.

For the derivations of $\alpha, \beta$ for Neg\kld, and the derivation of $\beta$ for Q-measure, the derivations follow when we work with definitions of these performance measures with the TP and TN \emph{counts} or cumulative surrogate reward values, rather than the TPR and TNR values and the average surrogate rewards.

\newcommand{\bl}{\text{\boldmath${\mathbf \ell}$}}
\newcommand{\pew}{\mathbf{p}}
\section{Proof of Theorem~\ref{THM:NEMSIS-ANALYSIS}}
\label{app:thm:nemsis-proof}

We begin by observing the following general lemma regarding the follow the leader algorithm for strongly convex losses. This will be useful since steps 15-17 of Algorithm~\ref{algo:nemsis} are essentially executing follow the leader steps to decide the best value for the dual variables.
\begin{lem}
\label{lem:ftl-reg}
Suppose we have an action space $\X$ and execute the follow the leader algorithm on a sequence of loss functions $\ell_t: \X \rightarrow \R$, each of which is $\alpha$-strongly convex and $L$-Lipschitz, then we have
\[
\sum_{t=1}^T\ell_t(\x_t) - \inf_{\x\in\X}\sum_{t=1}^T\ell_t(\x) \leq \frac{L^2\log T}{\alpha},
\]
where $\x_{t+1} = \underset{\x \in \X}{\arg\min}\sum_{\tau=1}^t\ell_\tau(\x)$ are the FTL plays.
\end{lem}
\begin{proof}
By the standard forward regret analysis, we get
\[
\sum_{t=1}^T\ell_t(\x_t) - \inf_{\x\in\X}\sum_{t=1}^T\ell_t(\x) \leq \sum_{t=1}^T\ell_t(\x_t) - \sum_{t=1}^T\ell_t(\x_{t+1})
\]
Now, by using the strong convexity of the loss functions, and the fact that the strong convexity property is additive, we get
\begin{align*}
\sum_{\tau=1}^{t-1} \ell_\tau(\x_{t+1}) &\geq \sum_{\tau=1}^{t-1} \ell_\tau(\x_t) + \frac{\alpha(t-1)}{2}\norm{\x_t - \x_{t+1}}_2^2\\
\sum_{\tau=1}^{t} \ell_\tau(\x_t) &\geq \sum_{\tau=1}^{t} \ell_\tau(\x_{t+1}) + \frac{\alpha t}{2}\norm{\x_t - \x_{t+1}}_2^2,
\end{align*}
which gives us
\[
\ell_t(\x_t) - \ell_t(\x_{t+1}) \geq \frac{\alpha}{2}(2t -1)\cdot\norm{\x_t - \x_{t+1}}_2^2.
\]
However, we get $\ell_t(\x_t) - \ell_t(\x_{t+1}) \leq L\cdot\norm{\x_t - \x_{t+1}}_2$ by invoking Lipschitz-ness of the loss functions. This gives us
\[
\norm{\x_t - \x_{t+1}}_2 \leq \frac{2L}{\alpha(2t-1)}.
\]
This, upon applying Lipschitz-ness again, gives us
\[
\ell_t(\x_t) - \ell_t(\x_{t+1}) \leq \frac{2L^2}{\alpha(2t-1)}.
\]
Summing over all the time steps gives us the desired result.
\end{proof}

For the rest of the proof, we shall use the shorthand notation that we used in Algorithm~\ref{algo:nemsis}, i.e.
\begin{align*}
\valpha_t &= (\alpha_{t,1},\alpha_{t,2}),\text{\hspace*{5ex}}&\text{(the dual variables for $\zeta_1$)}\\
\vbeta_t &= (\beta_{t,1},\beta_{t,2}),&\text{(the dual variables for $\zeta_2$)}\\
\vgamma_t &= (\gamma_{t,1},\gamma_{t,2}),&\text{(the dual variables for $\Psi$)}
\end{align*}
We will also use additional notation
\begin{align*}
\vecs_t &= (r^+(\w_t;\x_t,y_t),r^-(\w_t;\x_t,y_t)),\\
\pew_t &= \br{\valpha_t^\top\vecs_t - \zeta_1^\ast(\valpha_t), \vbeta_t^\top\vecs_t - \zeta_2^\ast(\vbeta_t)},\\
\bl_t(\w) &= (r^+(\w;\x_t,y_t),r^-(\w;\x_t,y_t))\\
R(\w) &= (P(\w),N(\w))
\end{align*}

Note that $\bl_t(\w_t) = \vecs_t$. We now define a quantity that we shall be analyzing to obtain the convergence bound
\[
(A) = \sum_{t=1}^T(\vgamma_t^\top\pew_t - \Psi^\ast(\vgamma_t))
\]

Now, since $\Psi$ is $\beta_\Psi$-smooth and concave, by Theorem~\ref{thm:scss}, we know that $\Psi^\ast$ is $\frac{1}{\beta}$-strongly concave. However that means that the loss function $g_t(\vgamma) := \vgamma^\top\pew_t - \Psi^\ast(\vgamma)$ is $\frac{1}{\beta}$-strongly convex. Now Algorithm~\ref{algo:nemsis} (step 17) implements
\[
\vgamma_t = \underset{\vgamma}{\arg\min}\bc{\vgamma\cdot\qew_t - \Psi^\ast(\vgamma)},
\]
where $\qew_t = \frac{1}{t}\sum_{\tau=1}^t\pew_\tau$ (see steps 7, 10, 13 that update $\qew_t$). Notice that this is identical to the FTL algorithm with the losses $g_t(\vgamma) = \pew_t\cdot\vgamma - \Psi^\ast(\vgamma)$ which are strongly convex, and can be shown to be $(B_r(L_{\zeta_1}+L_{\zeta_2}))$-Lipschitz, neglecting universal constants. Thus, by an application of Lemma~\ref{lem:ftl-reg}, we get, upto universal constants
\[
(A) \leq \inf_{\vgamma}\bc{\sum_{t=1}^T(\vgamma^\top\pew_t - \Psi^\ast(\vgamma))} + \beta_\Psi (B_r(L_{\zeta_1}+L_{\zeta_2}))^2\log T
\]
The same technique, along with the observation that steps 15 and 16 of Algorithm~\ref{algo:nemsis} also implement the FTL algorithm, can be used to get the following results upto universal constants
\[
\sum_{t=1}^T(\valpha_t^\top\vecs_t - \zeta_1^\ast(\valpha_t)) \leq \inf_{\valpha}\bc{\sum_{t=1}^T(\valpha^\top\vecs_t - \zeta_1^\ast(\valpha))} + \beta_{\zeta_1} (B_rL_{\zeta_1})^2\log T,
\]
and
\[
\sum_{t=1}^T(\vbeta_t^\top\vecs_t - \zeta_2^\ast(\vbeta_t)) \leq \inf_{\vbeta}\bc{\sum_{t=1}^T(\vbeta^\top\vecs_t - \zeta_2^\ast(\vbeta))} + \beta_{\zeta_2} (B_rL_{\zeta_2})^2\log T.
\]
This gives us, for
\[
\Delta_1 = \beta_\Psi (B_r(L_{\zeta_1}+L_{\zeta_2}))^2 + \beta_{\zeta_1} (B_rL_{\zeta_1})^2 + \beta_{\zeta_2} (B_rL_{\zeta_2})^2,
\]
\begin{align*}
(A) \leq{}& \inf_{\vgamma}\bc{{{\sum_{t=1}^T(\vgamma^\top\pew_t} - \Psi^\ast(\vgamma))}} + \Delta_1\log T\\
		={}& \inf_{\vgamma}\bc{\gamma_1\sum_{t=1}^T(\valpha_t^\top\vecs_t - \zeta_1^\ast(\valpha_t)) + \gamma_2\sum_{t=1}^T(\vbeta_t^\top\vecs_t - \zeta_2^\ast(\vbeta_t)) - \Psi^\ast(\vgamma)} + \Delta_1\log T\\
		\leq{}& \inf_{\vgamma,\valpha,\vbeta}\bc{\gamma_1\sum_{t=1}^T(\valpha^\top\vecs_t - \zeta_1^\ast(\valpha)) + \gamma_2\sum_{t=1}^T(\vbeta^\top\vecs_t - \zeta_2^\ast(\vbeta)) - \Psi^\ast(\vgamma)} + \Delta_1\log T
\end{align*}\normalsize
Now, because of the stochastic nature of the samples, we have
\[
\E{\vecs_t|\bc{(\x_\tau,y_\tau)}_{\tau=1}^{t-1}} = \E{\bl_{t}(\w_t)|\bc{(\x_\tau,y_\tau)}_{\tau=1}^{t-1}} = R(\w_t)
\]

\begin{align*}
\frac{(A)}{T} \leq{}& \inf_{\vgamma}\left\{\gamma_1\inf_{\valpha}\bc{\valpha^\top\br{\frac{1}{T}\sum_{t=1}^TR(\w_t)} - \zeta_1^\ast(\valpha)} \right.\\
&\hspace*{7ex}+ \left\{\gamma_2\inf_{\vbeta}\bc{\vbeta^\top\br{\frac{1}{T}\sum_{t=1}^TR(\w_t)} - \zeta_2^\ast(\vbeta)} - \Psi^\ast(\vgamma)\right\}\\
&{}+ {\frac{\Delta_1}{T}\br{\log T + \log\frac{1}{\delta}}}\\
={}& \inf_{\vgamma}\br{\gamma_1\zeta_1\br{\frac{1}{T}\sum_{t=1}^TR(\w_t)} + \gamma_2\zeta_2\br{\frac{1}{T}\sum_{t=1}^TR(\w_t)} - \Psi^\ast(\vgamma)}\\
&{}+ {\frac{\Delta_1}{T}\br{\log T + \log\frac{1}{\delta}}}\\
\leq{}& \inf_{\vgamma}\br{\gamma_1\zeta_1(R(\barw)) + \gamma_2\zeta_2(R(\barw)) - \Psi^\ast(\vgamma)} + {\frac{\Delta_1\log\frac{T}{\delta}}{T}}\\
={}& \Psi(\zeta_1(R(\barw)), \zeta_2(R(\barw))) + {\frac{\Delta_1\log\frac{T}{\delta}}{T}},
\end{align*}\normalsize
where the second last step follows from the Jensen's inequality, the concavity of the functions $P(\w)$ and $N(\w)$, and the assumption that $\zeta_1$ and $\zeta_2$ are increasing functions of both their arguments. Thus, we have, with probability at least $1 - \delta$,
\[
(A) \leq T\cdot\Psi(\zeta_1(R(\barw)),\zeta_2(R(\barw))) + {\Delta_1\log\frac{T}{\delta}}
\]
Note that this is a much stronger bound than what Narasimhan \etal \cite{NarasimhanKJ2015} obtain for their gradient descent based dual updates. This, in some sense, establishes the superiority of the follow-the-leader type algorithms used by \nemsis.

\scriptsize
\[
h_t(\w) = \gamma_{t,1}(\valpha_t^\top\bl_t(\w) - \zeta_1^\ast(\valpha_t)) + \gamma_{t,2}(\vbeta_t^\top\bl_t(\w)  - \zeta_2^\ast(\vbeta_t)) - \Psi^\ast(\vgamma_t)
\]\normalsize

Since the functions $h_t(\cdot)$ are concave and $(L_\Psi L_r(L_{\zeta_1}+L_{\zeta_2}))$-Lipschitz (due to assumptions on the smoothness and values of the reward functions), the standard regret analysis for online gradient ascent (for example \cite{zinkevich}) gives us the following bound on $(A)$, ignoring universal constants

\begin{align*}
(A) ={}& \sum_{t=1}^Th_t(\w_t)\\
		={}& \sum_{t=1}^T\gamma_{t,1}(\valpha_t^\top\bl_t(\w_t) - \zeta_1^\ast(\valpha_t)) + \gamma_{t,2}(\vbeta_t^\top\bl_t(\w_t) - \zeta_2^\ast(\vbeta_t)) - \Psi^\ast(\vgamma_t)\\
		\geq{}& \sum_{t=1}^T\gamma_{t,1}(\valpha_t^\top\bl_t(\w^\ast) - \zeta_1^\ast(\valpha_t)) + \gamma_{t,2}(\vbeta_t^\top\bl_t(\w^\ast) - \zeta_2^\ast(\vbeta_t))\\
		&{}	- \Psi^\ast(\vgamma_t) - \Delta_2{\sqrt T},
\end{align*}\normalsize
where $\Delta_2 = (L_\Psi L_r(L_{\zeta_1}+L_{\zeta_2}))$. Note that the above results hold since we used step lengths $\eta_t = \Theta(1/\sqrt t)$. To achieve the above bounds precisely, $\eta_t$ will have to be tuned to the Lipschitz constant of the functions $h_t(\cdot)$ and for sake of simplicity we assume that the step lengths are indeed tuned so. We also assume, to get the above result , without loss of generality of course, that the model space $\W$ is the unit norm ball in $\R^d$. Applying a standard online-to-batch conversion bound (for example \cite{online-batch-single}), then gives us, with probability at least $1 - \delta$,

\begin{align*}
\frac{(A)}{T} \geq{}& \underbrace{\frac{1}{T}\sum_{t=1}^T\gamma_{t,1}(\valpha_t^\top R(\w^\ast) - \zeta_1^\ast(\valpha_t))}_{(B)} + \underbrace{\frac{1}{T}\sum_{t=1}^T\gamma_{t,2}(\vbeta_t^\top R(\w^\ast) - \zeta_2^\ast(\vbeta_t))}_{(C)}\\
 &{}- \frac{1}{T}\sum_{t=1}^T\Psi^\ast(\vgamma_t) - \Delta_3{\frac{\log\frac{1}{\delta}}{\sqrt T}},
\end{align*}\normalsize
where $\Delta_3 = \Delta_2 + L_\Psi B_r(L_{\zeta_1}+L_{\zeta_2})$. Analyzing the expression $(B)$ 	gives us

\begin{align*}
(B) &= \frac{1}{T}\sum_{t=1}^T\gamma_{t,1}(\valpha_t^\top R(\w^\ast) - \zeta_1^\ast(\valpha_t))\\
		&= \frac{\sum_{t=1}^T\gamma_{t,1}}{T}\br{\br{\sum_{t=1}^T\frac{\gamma_{t,1}\valpha_t}{\sum_{t=1}^T\gamma_{t,1}}}^\top R(\w^\ast) - \sum_{t=1}^T\frac{\gamma_{t,1}}{\sum_{t=1}^T\gamma_{t,1}}\zeta_1^\ast(\valpha_t)}\\
		&\geq \frac{\sum_{t=1}^T\gamma_{t,1}}{T}\br{\br{\sum_{t=1}^T\frac{\gamma_{t,1}\valpha_t}{\sum_{t=1}^T\gamma_{t,1}}}^\top R(\w^\ast) - \zeta_1^\ast\br{\sum_{t=1}^T\frac{\gamma_{t,1}\valpha_t}{\sum_{t=1}^T\gamma_{t,1}}}}\\
		&\geq \frac{\sum_{t=1}^T\gamma_{t,1}}{T}\min_{\valpha}\bc{\valpha^\top R(\w^\ast) - \zeta_1^\ast(\valpha)}\\
		&= \bar\gamma_1\min_{\valpha}\bc{\valpha^\top R(\w^\ast) - \zeta_1^\ast(\valpha)} = \bar\gamma_1\zeta_1(R(\w^\ast))
\end{align*}\normalsize
A similar analysis for $(C)$ follows and we get, ignoring universal constants,

\begin{align*}
\frac{(A)}{T} \geq{}& \bar\gamma_1\zeta_1(R(\w^\ast)) + \bar\gamma_2\zeta_2(R(\w^\ast)) - \frac{1}{T}\sum_{t=1}^T\Psi^\ast(\vgamma_t) - \Delta_3{\frac{\log\frac{1}{\delta}}{\sqrt T}}\\
							\geq{}& \bar\gamma_1\zeta_1(R(\w^\ast)) + \bar\gamma_2\zeta_2(R(\w^\ast)) - \Psi^\ast(\bar\vgamma) - \Delta_3{\frac{\log\frac{1}{\delta}}{\sqrt T}}\\
							\geq{}&\min_{\vgamma}\left\{\gamma_1\zeta_1(R(\w^\ast)) + \gamma_1\zeta_2(R(\w^\ast)) - \Psi^\ast(\vgamma)\right\} - \Delta_3{\frac{\log\frac{1}{\delta}}{\sqrt T}}\\
							={}& \Psi(\zeta_1(R(\w^\ast)), \zeta_2(R(\w^\ast))) - \Delta_3{\frac{\log\frac{1}{\delta}}{\sqrt T}}
\end{align*}\normalsize
Thus, we have with probability at least $1 - \delta$,
\[
(A) \geq T\cdot \Psi(\zeta_1(R(\w^\ast)), \zeta_2(R(\w^\ast))) - \Delta_3{\log\frac{1}{\delta}\sqrt T}
\]
Combining the upper and lower bounds on $(A)$ finishes the proof since $\Delta_3{\log\frac{1}{\delta}\sqrt T}$ overwhelms the term ${\Delta_1\log\frac{T}{\delta}}$.

\newcommand{\drp}{\Pf_\quant}
\newcommand{\nrp}{\Pf_\class}
\newcommand{\tw}{\tilde\w}
\newcommand{\wn}{\w_{t+1}}
\section{Proof of Theorem~\ref{THM:SCANFC-CONV}}
\label{app:scanfcconv}

\noindent We will prove the result by proving a sequence of claims. The first claim ensures that the distance to the optimum performance value is bounded by the performance value we obtain in terms of the valuation function at any step. For notational simplicity, we will use the shorthand $\Pf(\w) := \Pf_{(\Pf_\quant,\Pf_\class)}(\w)$ .
\begin{clm}
\label{clm:boundprog}
$\Pf^\ast := \sup_{\w\in\W}\Pf(\w)$ be the optimal performance level. Also, define $e_t = V(\w_{t+1},v_t)$. Then, for any $t$, we have
\[
\Pf^\ast - \Pf(\w_t) \leq \frac{e_t}{m}
\]
\end{clm}
\begin{proof}
We will prove the result by contradiction. Suppose $\Pf^\ast > \Pf(\w_t) + \frac{e_t}{m}$. Then there must exist some $\tw \in \W$ such that
\[
\Pf(\tw) = \frac{e_t}{m} + \Pf(\w_t) + e' = \frac{e_t}{m} + v_t + e' =: v',
\]
where $e' > 0$. Note that the above uses the fact that we set $v_t = \Pf(\w_t)$. Then we have
\[
V(\tw,v_t) - e_t = \nrp(\tw) - v_t\cdot\drp(\tw) - e_t.
\]
Now since $\Pf(\tw) = v'$, we have $\nrp(\tw) - v'\cdot\drp(\tw) = 0$ which gives us
\[
V(\tw,v_t) - e_t = \br{\frac{e_t}{m} + e'}\drp(\tw) - e_t \geq \br{\frac{e_t}{m} + e'}m - e_t > 0.
\]
But this contradicts the fact that $\max_{\w \in \W}V(\w,v_t) = e_t$ which is ensured by step 4 of Algorithm~\ref{alg:scanfc}. This completes the proof.
\end{proof}

The second claim then establishes that in case we do get a large performance value in terms of the valuation function at any time step, the next iterate will have a large leap in performance in terms of the original performance function $\Pf$.

\begin{clm}
\label{clm:suffprog}
For any time instant $t$ we have
\[
\Pf(\w_{t+1}) \geq \Pf(\w_t) + \frac{e_t}{M}
\]
\end{clm}
\begin{proof}
By our definition, we have $V(\w_{t+1},v_t) = e_t$. This gives us
\begin{align*}
\frac{\nrp(\wn)}{\drp(\wn)} - \br{v_t + \frac{e_t}{M}} &\geq \frac{\nrp(\wn)}{\drp(\wn)} - \br{v_t + \frac{e_t}{\drp(\wn)}}\\
&= \frac{v_t\cdot\drp(\wn)}{\drp(\wn)} - v_t = 0,
\end{align*}
which proves the result.
\end{proof}

We are now ready to establish the convergence proof. Let $\Delta_t = \Pf^\ast - \Pf(\w_t)$. Then we have, by Claim~\ref{clm:boundprog}
\[
e_t \geq m\cdot\Delta_t,
\]
and also
\[
\Pf(\wn) \geq \Pf(\w_t) + \frac{e_t}{M},
\]
by Claim~\ref{clm:suffprog}. Subtracting both sides of the above equation from $\Pf^\ast$ gives us
\begin{align*}
\Delta_{t+1} &= \Delta_t - \frac{e_t}{M}\\
&\leq \Delta_t - \frac{m}{M}\cdot \Delta_t = \br{1 - \frac{m}{M}}\cdot \Delta_t,
\end{align*}
which concludes the convergence proof.

\section{Proof of Theorem~\ref{THM:SCAN-CONV}}
\label{app:scanconv}

To prove this theorem, we will first show that the \scanfc algorithm is robust to imprecise updates. More precisely, we will assume that Algorithm~\ref{alg:scanfc} only ensures that in step 4 we have
\[
V(\w_{t+1},v_t) = \max_{\w\in\W} V(\w,v_t) - \epsilon_t,
\]
where $\epsilon_t > 0$ and step 5 only ensures that
\[
v_{t} = \Pf(\w_t) + \delta_t,
\]
where $\delta_t$ may be positive or negative. For this section, we will redefine
\[
e_t = \max_{\w\in\W}V(\w,v_t)
\]
since we can no longer assume that $V(\wn,v_t) = e_t$. Note that if $v_t$ is an unrealizable value, i.e. for no predictor $\w \in \W$ is $\Pf(\w) \geq v_t$, then we have $e_t < 0$. Having this we establish the following results:

\begin{lem}
\label{lem:scanfc-robust}
Given the previous assumptions on the imprecise execution of Algorithm~\ref{alg:scanfc}, the following is true
\begin{enumerate}
	\item If $\delta_t \leq 0$ then $e_t \geq 0$
	\item If $\delta_t > 0$ then $e_t \geq -\delta_t\cdot M$
	\item We have $\Pf^\ast < v_t$ iff $e_t < 0$
	\item If $e_t \geq 0 $ then $e_t \geq m(\Pf^\ast - v_t)$
	\item If $e_t < 0 $ then $e_t \geq M(\Pf^\ast - v_t)$
	\item If $V(\w,v) = e$ for $e \geq 0$ then $\Pf(\w) \geq v + \frac{e}{M}$
	\item If $V(\w,v) = e$ for $e < 0$ then $\Pf(\w) \geq v + \frac{e}{m}$
\end{enumerate}
\end{lem}
\begin{proof}
We prove the parts separately below
\begin{enumerate}
	\item Since $\delta_t < 0$, there exists some $\w \in \W$ such that $\Pf(\w) > v_t$. The result then follows.
	\item If $v_t = \Pf(\w_t) + \delta_t$ then $V(\w_t,v_t) \geq -\delta_t\cdot M$ .The result then follows.
	\item Had $e_t \geq 0$ been the case, we would have had, for some $\w \in \W$, $V(\w,v_t) \geq 0$ which would have implied $\Pf(\w) \geq v_t$ which contradicts $\Pf^\ast < v_t$. For the other direction, suppose $\Pf^\ast = \Pf(\w^\ast) = v_t + e'$ with $e' > 0$. Then we have $e_t = V(\w^\ast,v_t) > 0$ which contradicts $e_t < 0$.
	\item Observe that the proof of Claim~\ref{clm:boundprog} suffices, by simply replacing $\Pf(\w_t)$ with $v_t$ in the statement.
	\item Assume the contrapositive that for some $\w \in \W$, we have $\Pf(\w) = v_t + \frac{e_t}{M} + e'$ where $e' > 0$. We can then show that $V(\w,v_t) = \br{\frac{e_t}{M} + e'}\drp(\w) \geq e_t + e'\cdot\drp(\w) > e_t$ which contradicts the definition of $e_t$. Note that since $e_t < 0$, we have $\frac{e_t}{M} \geq \frac{e_t}{\drp(\w)}$ and we have $\drp(\w) \geq m > 0$.
	\item Observe that the proof of Claim~\ref{clm:suffprog} suffices, by simply replacing $\Pf(\w_t)$ with $v_t$ in the statement.
	\item We have $\nrp(\w) - v\cdot\drp(\w) = e$. Dividing throughout by $\drp(\w) > 0$ and using $\frac{e}{\drp(\w)} \geq \frac{e}{m}$ since $e < 0$ gives us the result.
\end{enumerate}
This finishes the proofs.
\end{proof}

Using these results, we can now make the following claim on the progress made by \scanfc with imprecise updates.

\begin{lem}
\label{lem:scanfc-noisy-progress}
Even if \scanfc is executed with noisy updates, at any time step $t$, we have
\[
\Delta_{t+1} \leq \br{1 - \frac{m}{M}} \Delta_t + \frac{M}{m}\cdot\abs{\delta_t} + \frac{\epsilon_t}{m}.
\]
\end{lem}
\begin{proof}
We analyze time steps when $\delta_t \leq 0$ separately from time steps when $\delta_t < 0$.\\

\noindent\textbf{Case 1}: $\boldsymbol{\delta_t\leq 0}$ In these time steps, the method underestimates the performance of the current predictor but gives a legal i.e. realizable value of $v_t$. We first deduce that for these time steps, using Lemma~\ref{lem:scanfc-robust} part 1, we have $e_t \geq 0$ and then using part 4, we have $e_t \geq m(\Pf^\ast - v_t)$. This combined with the identity $\Pf^\ast - v_t = \Delta_t - \delta_t$, gives us
\[
e_t \geq m(\Delta_t - \delta_t)
\]
Now we have, by definition, $V(\wn,v_t) = e_t - \epsilon_t$ (note that both $e_t, \epsilon \geq 0$ in this case). The next steps depend on whether this quantity is positive or negative. If $\epsilon_t \leq e_t$, we apply Lemma~\ref{lem:scanfc-robust} part 6 to get
\[
\Pf(\wn) \geq v_t + \frac{e_t - \epsilon_t}{M},
\]
which gives us upon using $\Pf^\ast - v_t = \Delta_t - \delta_t$ and $e_t \geq m(\Delta_t - \delta_t)$,
\[
\Delta_{t+1} \leq \br{1 - \frac{m}{M}} \Delta_t - \br{1 - \frac{m}{M}} \delta_t + \frac{\epsilon_t}{M}
\]
Otherwise if $\epsilon_t > e_t$ then we have actually made negative progress at this time step since $V(\wn,v_t) < 0$. To safeguard us against how much we go back in terms of progress, we us Lemma~\ref{lem:scanfc-robust} part 7 to guarantee
\[
\Pf(\wn) \geq v_t + \frac{e_t - \epsilon_t}{m},
\]
which gives us upon using $\Pf^\ast - v_t = \Delta_t - \delta_t$ and $e_t \geq m(\Delta_t - \delta_t)$,
\[
\Delta_{t+1} \leq \frac{\epsilon_t}{m},
\]
Note however, that we are bound by $\epsilon_t > e_t$ in the above statement. We now move on to analyze the second case.\\

\noindent\textbf{Case 2}: $\boldsymbol{\delta_t > 0}$ In these time steps, the method is overestimating the performance of the current predictor and runs a risk of giving a value of $v_t$ that is unrealizable. We cannot hope to make much progress in these time steps. The following analysis simply safeguards us against too much deterioration. There are two subcases we explore here: first we look at the case where $v_t \leq \Pf^\ast$ i.e. $v_t$ is still a legal, realizable performance value. In this case we continue to have $e_t \geq 0$ and the analysis of the previous case (i.e. $\delta_t \leq 0$) continues to apply.

However, if $v_t > \Pf^\ast$, we are setting an unrealizable value of $v_t$. Using Lemma~\ref{lem:scanfc-robust} part 3 gives us $e_t < 0$ which, upon using part 5 of the lemma gives us
\[
e_t \geq M(\Pf^\ast - v_t).
\]
In this case, we have $V(\wn,v_t) = e_t - \epsilon_t < 0$ since $e_t < 0$ and $\epsilon_t > 0$. Thus, using Lemma~\ref{lem:scanfc-robust} part 7 gives us
\[
\Pf(\wn) \geq v_t + \frac{e_t - \epsilon_t}{m}
\]
which upon manipulation, as before, gives us
\[
\Delta_{t+1} \leq \br{1 - \frac{M}{m}}\Delta_t + \br{\frac{M}{m} - 1}\delta_t + \frac{\epsilon_t}{m} \leq \br{\frac{M}{m} - 1}\delta_t + \frac{\epsilon_t}{m},
\]
where the last step uses the fact that $\Delta_t \geq 0$ and $M \geq m$. Putting all these cases together and using the fact that the quantities $\Delta_t, \epsilon_t, \abs{\delta_t}$ are always positive gives us
\begin{align*}
\Delta_{t+1} &\leq \br{1 - \frac{m}{M}} \Delta_t + \br{\frac{M^2 - m^2}{mM}} \abs{\delta_t} + \frac{\epsilon_t}{m}\\
						 &\leq \br{1 - \frac{m}{M}} \Delta_t + \frac{M}{m}\cdot\abs{\delta_t} + \frac{\epsilon_t}{m},
\end{align*}
which finishes the proof.
\end{proof}

From hereon simple manipulations similar to those used to analyze the \amsgd algorithm in \cite{NarasimhanKJ2015} can be used, along with the guarantees provided by Theorem~\ref{THM:NEMSIS-ANALYSIS} for the \nemsis analysis to finish the proof of the result. We basically have to use the fact that the \nemsis invocations in \scan (Algorithm~\ref{alg:scan} line 8), as well as the performance estimation steps (Algorithm~\ref{alg:scan} lines 14-19) can be seen as executing noisy updates for the original \scanfc algorithm.

\end{document}